\documentclass[10pt,twocolumn,letterpaper]{article}

\usepackage{style.files/cvpr,times,epsfig,graphicx,amsmath,amssymb}
\usepackage{amsthm,fixmath,tabulary,mathtools,xspace,multirow}
\usepackage[pagebackref=true,breaklinks=true,letterpaper=true,colorlinks,
  citecolor=citecolor,bookmarks=false]{hyperref}
\definecolor{citecolor}{RGB}{34,139,34}
\def\cvprPaperID{}\cvprfinalcopy

\usepackage{subcaption}
\newlength\savewidth\newcommand\shline{\noalign{\global\savewidth\arrayrulewidth
  \global\arrayrulewidth 1pt}\hline\noalign{\global\arrayrulewidth\savewidth}}
\newcommand{\tablestyle}[2]{\setlength{\tabcolsep}{#1}\renewcommand{\arraystretch}{#2}\centering\small}
\newcommand{\eqnsm}[2]{\begin{equation}\label{eq:#1}#2\end{equation}}
\newcommand{\tpm}[2]{\resizebox{4mm}{!}{\Large\raisebox{1mm}{$\genfrac{}{}{0pt}{}{+#1}{-#2}$}}}
\newcommand{\tss}[1]{\textsuperscript{#1}}

\makeatletter
\renewcommand\paragraph{\@startsection{paragraph}{4}{\z@}%
  {.1em \@plus1ex \@minus.1ex}%
  {-1em}%
  {\normalfont\normalsize\bfseries}}
\makeatother
\newtheorem{theorem}{Theorem}

\newcommand{\TP}{\mathit{TP}}
\newcommand{\FN}{\mathit{FN}}
\newcommand{\FP}{\mathit{FP}}
\newcommand{\IoU}{\text{IoU}}
\newcommand{\BN}{\mathbb N}
\newcommand{\SL}{{\cal L}}
\newcommand{\things}{\tss{Th}\xspace}
\newcommand{\stuff}{\tss{St}\xspace}

\begin{document}

\title{Panoptic Segmentation}
\author{
 Alexander Kirillov$^{1,2}$ \quad Kaiming He$^1$ \quad Ross Girshick$^1$
 \quad Carsten Rother$^2$ \quad Piotr Doll\'ar$^1$\\[2mm]
 $^1$Facebook AI Research (FAIR) \qquad $^2$HCI/IWR, Heidelberg University, Germany
}
\maketitle

\begin{abstract}
We propose and study a task we name \emph{panoptic segmentation} (PS). Panoptic segmentation unifies the typically distinct tasks of \emph{semantic segmentation} (assign a class label to each pixel) and \emph{instance segmentation} (detect and segment each object instance). The proposed task requires generating a \emph{coherent} scene segmentation that is rich and complete, an important step toward real-world vision systems. While early work in computer vision addressed related image/scene parsing tasks, these are not currently popular, possibly due to lack of appropriate metrics or associated recognition challenges. To address this, we propose a novel \emph{panoptic quality} (PQ) metric that captures performance for all classes (stuff and things) in an interpretable and unified manner. Using the proposed metric, we perform a rigorous study of both human and machine performance for PS on three existing datasets, revealing interesting insights about the task. The aim of our work is to revive the interest of the community in a more unified view of image segmentation.
\end{abstract}

\section{Introduction}\label{sec:intro}

\begin{figure}
\captionsetup[subfigure]{aboveskip=0mm,belowskip=.4mm}
\begin{subfigure}{0.495\linewidth}
 \includegraphics[width=1.0\linewidth,trim=2cm 0 1cm 0,clip]{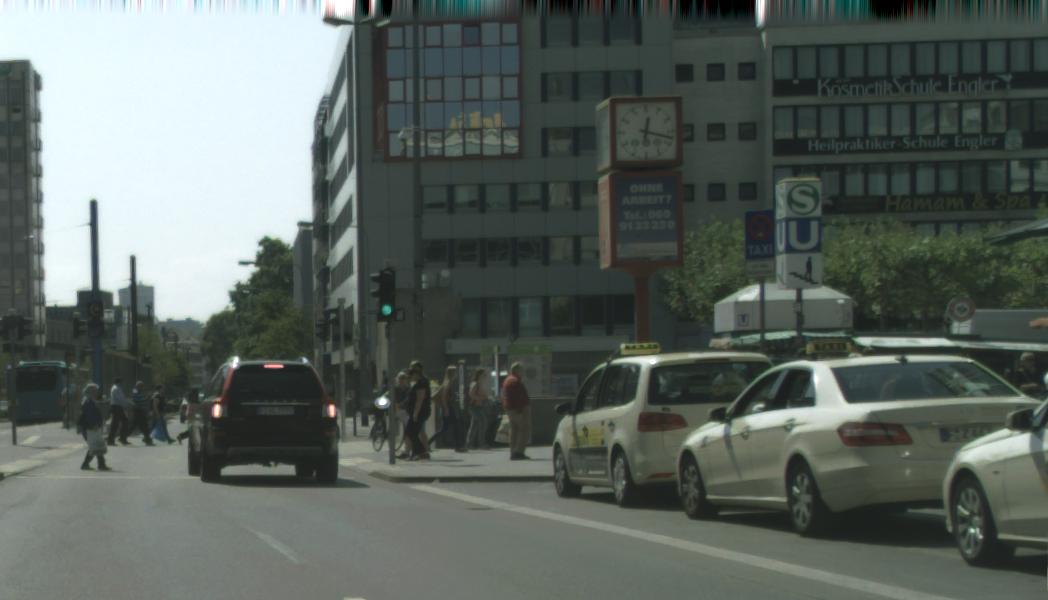}
 \subcaption{image}\label{fig:image}
\end{subfigure}
\begin{subfigure}{0.495\linewidth}
 \includegraphics[width=1.0\linewidth,trim=2cm 0 1cm 0,clip]{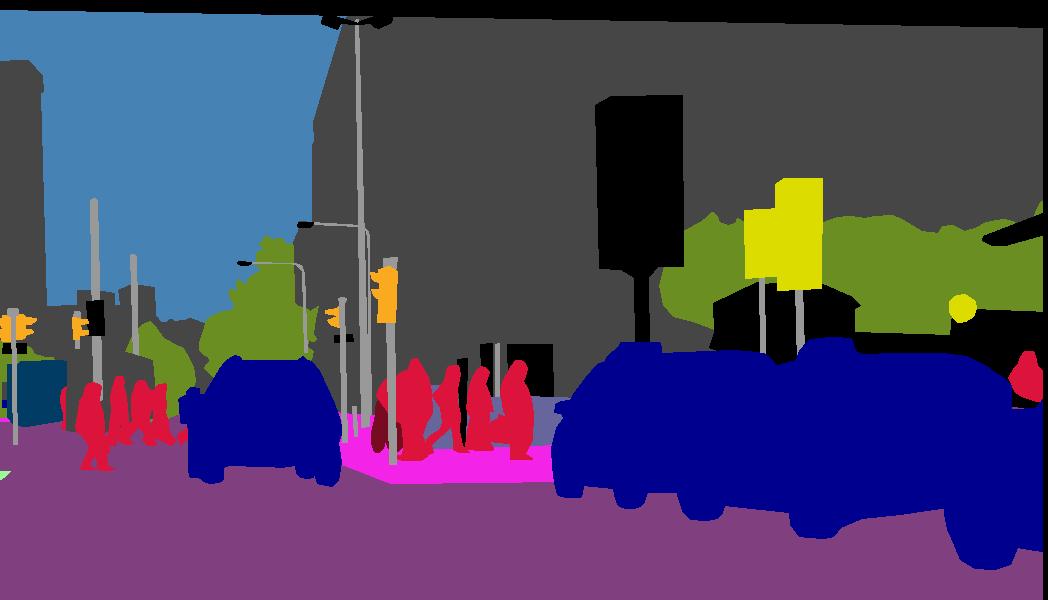}
 \subcaption{semantic segmentation}\label{fig:semantic}
\end{subfigure}\\
\begin{subfigure}{0.495\linewidth}
 \includegraphics[width=1.0\linewidth,trim=2cm 0 1cm 0,clip]{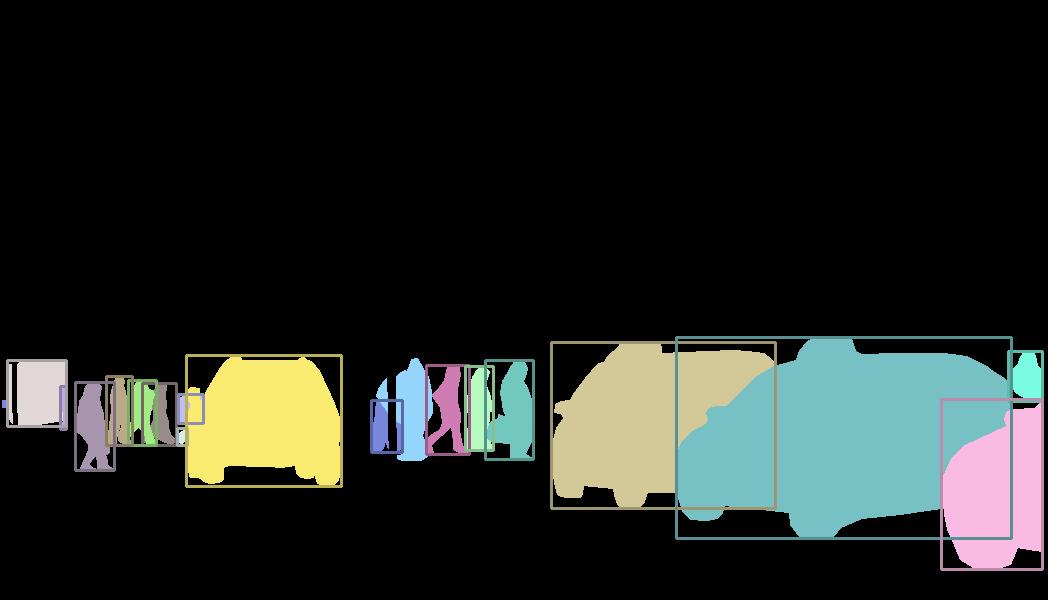}
 \subcaption{instance segmentation}\label{fig:instance}
\end{subfigure}
\begin{subfigure}{0.495\linewidth}
 \includegraphics[width=1.0\linewidth,trim=2cm 0 1cm 0,clip]{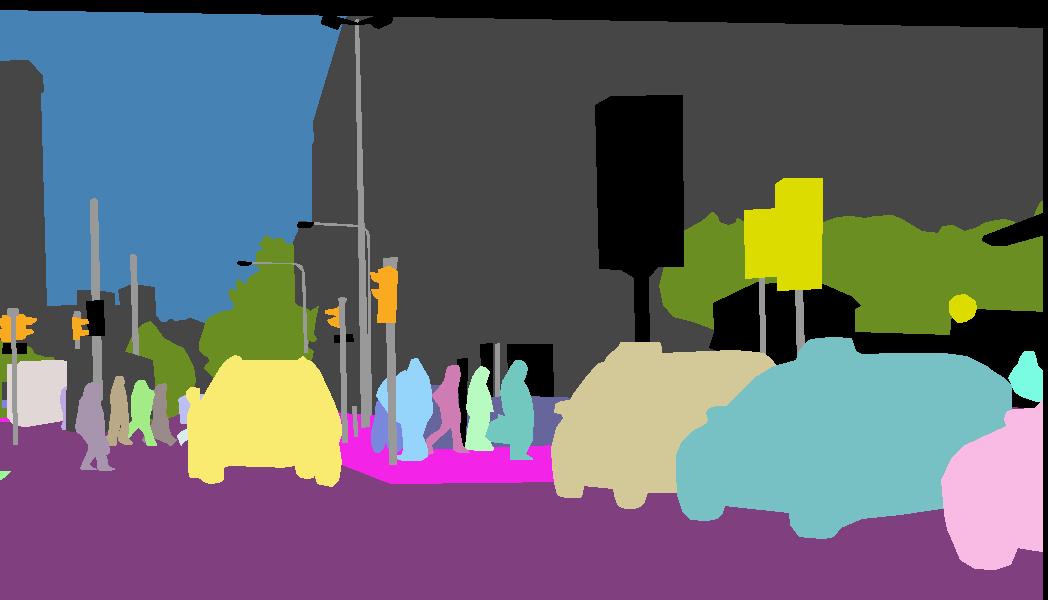}
 \subcaption{panoptic segmentation}\label{fig:panoptic}
\end{subfigure}
\caption{For a given (\subref{fig:image}) image, we show \emph{ground truth} for: (\subref{fig:semantic}) semantic segmentation (per-pixel class labels), (\subref{fig:instance}) instance segmentation (per-object mask and class label), and (\subref{fig:panoptic}) the proposed \emph{panoptic segmentation} task (per-pixel class+instance labels). The PS task: (1) encompasses both stuff and thing classes, (2) uses a simple but general format, and (3) introduces a uniform evaluation metric for all classes. Panoptic segmentation generalizes both semantic and instance segmentation and we expect the unified task will present novel challenges and enable innovative new methods.}
\label{fig:examples}
\end{figure}

In the early days of computer vision, \emph{things} -- countable objects such as people, animals, tools -- received the dominant share of attention. Questioning the wisdom of this trend, Adelson \cite{adelson2001seeing} elevated the importance of studying systems that recognize \emph{stuff} -- amorphous regions of similar texture or material such as grass, sky, road. This dichotomy between stuff and things persists to this day, reflected in both the division of visual recognition tasks and in the specialized algorithms developed for stuff and thing tasks.

Studying stuff is most commonly formulated as a task known as \emph{semantic segmentation}, see Figure \ref{fig:semantic}. As stuff is amorphous and uncountable, this task is defined as simply assigning a class label to each pixel in an image (note that semantic segmentation treats thing classes as stuff). In contrast, studying things is typically formulated as the task of \emph{object detection} or \emph{instance segmentation}, where the goal is to detect each object and delineate it with a bounding box or segmentation mask, respectively, see Figure \ref{fig:instance}. While seemingly related, the datasets, details, and metrics for these two visual recognition tasks vary substantially.

The schism between semantic and instance segmentation has led to a parallel rift in the methods for these tasks. Stuff classifiers are usually built on fully convolutional nets \cite{long2015fully} with dilations \cite{yu2015multi,chen2016deeplab} while object detectors often use object proposals \cite{Hosang2015proposals} and are region-based \cite{Ren2015a,he2017mask}. Overall algorithmic progress on these tasks has been incredible in the past decade, yet, something important may be overlooked by focussing on these tasks in isolation.

A natural question emerges: \emph{Can there be a reconciliation between stuff and things?} And what is the most effective design of a unified vision system that generates rich and coherent scene segmentations? These questions are particularly important given their relevance in real-world applications, such as autonomous driving or augmented reality.

Interestingly, while semantic and instance segmentation dominate current work, in the pre-deep learning era there was interest in the joint task described using various names such as \emph{scene parsing} \cite{tighe2014scene}, \emph{image parsing} \cite{tu2005image}, or \emph{holistic scene understanding} \cite{yao2012describing}. Despite its practical relevance, this general direction is not currently popular, perhaps due to lack of appropriate metrics or recognition challenges.

In our work we aim to revive this direction. We propose a task that: \emph{(1) encompasses both stuff and thing classes, (2) uses a simple but general output format, and (3) introduces a uniform evaluation metric}. To clearly disambiguate with previous work, we refer to the resulting task as \emph{panoptic segmentation} (PS). The definition of `panoptic' is ``including everything visible in one view'', in our context panoptic refers to a unified, global view of segmentation.

The \textbf{task format} we adopt for panoptic segmentation is simple: each pixel of an image must be assigned a semantic label and an instance id. Pixels with the same label and id belong to the same object; for stuff labels the instance id is ignored. See Figure \ref{fig:panoptic} for a visualization. This format has been adopted previously, especially by methods that produce non-overlapping instance segmentations \cite{kirillov2016instancecut,liu2017sgn,arnab2017pixelwise}. We adopt it for our joint task that includes stuff and things.

A fundamental aspect of panoptic segmentation is the \textbf{task metric} used for evaluation. While numerous existing metrics are popular for either semantic or instance segmentation, these metrics are best suited either for stuff or things, respectively, but not both. We believe that the use of disjoint metrics is one of the primary reasons the community generally studies stuff and thing segmentation in isolation. To address this, we introduce the {\em panoptic quality} (PQ) metric in \S\ref{sec:metric}. PQ is \emph{simple} and \emph{informative} and most importantly can be used to measure the performance for both stuff and things in a \emph{uniform} manner. Our hope is that the proposed joint metric will aid in the broader adoption of the joint task.

The panoptic segmentation task encompasses both semantic and instance segmentation but introduces new algorithmic challenges. Unlike semantic segmentation, it requires differentiating individual object instances; this poses a challenge for fully convolutional nets. Unlike instance segmentation, object segments must be \emph{non-overlapping}; this presents a challenge for region-based methods that operate on each object independently. Generating coherent image segmentations that resolve inconsistencies between stuff and things is an important step toward real-world uses.

As both the ground truth and algorithm format for PS must take on the same form, we can perform a detailed study of \emph{human consistency} on panoptic segmentation. This allows us to understand the PQ metric in more detail, including detailed breakdowns of recognition \vs segmentation and stuff \vs things performance. Moreover, measuring human PQ helps ground our understanding of machine performance. This is important as it will allow us to monitor performance saturations on various datasets for PS.

Finally we perform an initial study of machine performance for PS. To do so, we define a simple and likely suboptimal heuristic that combines the output of two \emph{independent} systems for semantic and instance segmentation via a series of post-processing steps that merges their outputs (in essence, a sophisticated form of non-maximum suppression). Our heuristic establishes a baseline for PS and gives us insights into the main algorithmic challenges it presents.

We study both human and machine performance on three popular segmentation datasets that have both stuff and things annotations. This includes the Cityscapes \cite{Cordts2016Cityscapes}, ADE20k \cite{zhou2017ade20k}, and Mapillary Vistas \cite{neuhold2017mapillary} datasets. For each of these datasets, we obtained results of state-of-the-art methods directly from the challenge organizers. In the future we will extend our analysis to COCO \cite{lin2014coco} on which stuff is being annotated \cite{caesar2016coco}. Together our results on these datasets form a solid foundation for the study of both human and machine performance on panoptic segmentation.

Both COCO \cite{lin2014coco} and Mapillary Vistas \cite{neuhold2017mapillary} featured the panoptic segmentation task as one of the tracks in their recognition challenges at ECCV 2018. We hope that having PS featured alongside the instance and semantic segmentation tracks on these popular recognition datasets will help lead to a broader adoption of the proposed joint task.

\section{Related Work}

Novel datasets and tasks have played a key role throughout the history of computer vision. They help catalyze progress and enable breakthroughs in our field, and just as importantly, they help us measure and recognize the progress our community is making. For example, ImageNet \cite{Russakovsky2015} helped drive the recent popularization of deep learning techniques for visual recognition \cite{Krizhevsky2012} and exemplifies the potential transformational power that datasets and tasks can have. Our goals for introducing the panoptic segmentation task are similar: to challenge our community, to drive research in novel directions, and to enable both expected and unexpected innovation. We review related tasks next.

\paragraph{Object detection tasks.} Early work on face detection using ad-hoc datasets (\eg, \cite{Lecun94,Viola2001}) helped popularize bounding-box object detection. Later, pedestrian detection datasets \cite{Dollar2012PAMI} helped drive progress in the field. The PASCAL VOC dataset \cite{everingham2015pascal} upgraded the task to a more diverse set of general object classes on more challenging images. More recently, the COCO dataset \cite{lin2014coco} pushed detection towards the task of instance segmentation. By framing this task and providing a high-quality dataset, COCO helped define a new and exciting research direction and led to many recent breakthroughs in instance segmentation \cite{pinheiro2015learning,li2016fully,he2017mask}. Our general goals for panoptic segmentation are similar.

\paragraph{Semantic segmentation tasks.} Semantic segmentation datasets have a rich history \cite{shotton2006textonboost,liu2011sift,everingham2015pascal} and helped drive key innovations (\eg, fully convolutional nets \cite{long2015fully} were developed using \cite{liu2011sift,everingham2015pascal}). These datasets contain both stuff and thing classes, but don't distinguish individual object instances. Recently the field has seen numerous new segmentation datasets including Cityscapes \cite{Cordts2016Cityscapes}, ADE20k \cite{zhou2017ade20k}, and Mapillary Vistas \cite{neuhold2017mapillary}. These datasets actually support both semantic and instance segmentation, and each has opted to have a separate track for the two tasks. Importantly, they contain all of the information necessary for PS. In other words, \emph{the panoptic segmentation task can be bootstrapped on these datasets without any new data collection.}

\paragraph{Multitask learning.} With the success of deep learning for many visual recognition tasks, there has been substantial interest in \emph{multitask learning} approaches that have broad competence and can solve multiple diverse vision problems in a single framework \cite{kokkinos2016ubernet,malik2016three,misra2016cross}. \Eg, UberNet \cite{kokkinos2016ubernet} solves multiple low to high-level visual tasks, including object detection and semantic segmentation, using a single network. While there is significant interest in this area, we emphasize that panoptic segmentation is \emph{not} a multitask problem but rather a single, \emph{unified} view of image segmentation. Specifically, the multitask setting allows for independent and potentially inconsistent outputs for stuff and things, while PS requires a single coherent scene segmentation.

\paragraph{Joint segmentation tasks.} In the pre-deep learning era, there was substantial interest in generating coherent scene interpretations. The seminal work on image parsing \cite{tu2005image} proposed a general bayesian framework to jointly model segmentation, detection, and recognition. Later, approaches based on graphical models studied consistent stuff and thing segmentation \cite{yao2012describing,tighe2013finding,tighe2014scene,sun2014relating}. While these methods shared a common motivation, there was no agreed upon task definition, and different output formats and varying evaluation metrics were used, including separate metrics for evaluating results on stuff and thing classes. In recent years this direction has become less popular, perhaps for these reasons.

In our work we aim to revive this general direction, but in contrast to earlier work, we focus on the task itself. Specifically, as discussed, PS: (1) addresses both stuff and thing classes, (2) uses a simple format, and (3) introduces a uniform metric for both stuff and things. Previous work on joint segmentation uses varying formats and disjoint metrics for evaluating stuff and things. Methods that generate non-overlapping instance segmentations \cite{kirillov2016instancecut,bai2016deep,liu2017sgn,arnab2017pixelwise} use the same format as PS, but these methods typically only address thing classes. By addressing both stuff and things, using a simple format, and introducing a uniform metric, we hope to encourage broader adoption of the joint task.

\paragraph{Amodal segmentation task.} In \cite{zhu2017amodal} objects are annotated \emph{amodally}: the full extent of each region is marked, not just the visible. Our work focuses on segmentation of all \emph{visible} regions, but an extension of panoptic segmentation to the amodal setting is an interesting direction for future work.

\section{Panoptic Segmentation Format}\label{sec:task}

\paragraph{Task format.} The format for panoptic segmentation is simple to define. Given a predetermined set of $L$ semantic classes encoded by $\SL := \{0, \ldots, L-1 \}$, the task requires a \emph{panoptic segmentation algorithm} to map each pixel $i$ of an image to a pair $(l_i, z_i) \in \SL \times \BN$, where $l_i$ represents the semantic class of pixel $i$ and $z_i$ represents its instance id. The $z_i$'s group pixels of the same class into distinct segments. Ground truth annotations are encoded identically. Ambiguous or out-of-class pixels can be assigned a special void label; \ie, not all pixels must have a semantic label.

\paragraph{Stuff and thing labels.} The semantic label set consists of subsets $\SL\stuff$ and $\SL\things$, such that $\SL = \SL\stuff \cup \SL\things$ and $\SL\stuff \cap \SL\things = \emptyset$. These subsets correspond to \emph{stuff} and \emph{thing} labels, respectively. When a pixel is labeled with $l_i \in \SL\stuff$, its corresponding instance id $z_i$ is irrelevant. That is, for stuff classes all pixels belong to the same instance (\eg, the same \emph{sky}). Otherwise, all pixels with the same $(l_i, z_i)$ assignment, where $l_i \in \SL\things$, belong to the same instance (\eg, the same \emph{car}), and conversely, all pixels belonging to a single instance must have the same $(l_i, z_i)$. The selection of which classes are stuff \vs things is a design choice left to the creator of the dataset, just as in previous datasets.

\paragraph{Relationship to semantic segmentation.} The PS task format is a strict generalization of the format for semantic segmentation. Indeed, both tasks require each pixel in an image to be assigned a semantic label. If the ground truth does not specify instances, or all classes are stuff, then the task formats are identical (although the task metrics differ). In addition, inclusion of thing classes, which may have multiple instances per image, differentiates the tasks.

\paragraph{Relationship to instance segmentation.} The instance segmentation task requires a method to segment each object instance in an image. However, it allows overlapping segments, whereas the panoptic segmentation task permits only one semantic label and one instance id to be assigned to each pixel. Hence, for PS, no overlaps are possible by construction. In the next section we show that this difference plays an important role in performance evaluation.

\paragraph{Confidence scores.} Like semantic segmentation, but unlike instance segmentation, we do \emph{not} require confidence scores associated with each segment for PS. This makes the panoptic task \emph{symmetric} with respect to humans and machines: both must generate the same type of image annotation. It also makes evaluating human consistency for PS simple. This is in contrast to instance segmentation, which is not easily amenable to such a study as human annotators do not provide explicit confidence scores (though a single precision/recall point may be measured). We note that confidence scores give downstream systems more information, which can be useful, so it may still be desirable to have a PS algorithm generate confidence scores in certain settings.

\section{Panoptic Segmentation Metric}\label{sec:metric}

In this section we introduce a new metric for panoptic segmentation. We begin by noting that existing metrics are specialized for either semantic or instance segmentation and cannot be used to evaluate the joint task involving both stuff and thing classes. Previous work on joint segmentation sidestepped this issue by evaluating stuff and thing performance using independent metrics (\eg \cite{yao2012describing,tighe2013finding,tighe2014scene,sun2014relating}). However, this introduces challenges in algorithm development, makes comparisons more difficult, and hinders communication. We hope that introducing a unified metric for stuff and things will encourage the study of the unified task.

Before going into further details, we start by identifying the following desiderata for a suitable metric for PS:

\textbf{Completeness.} The metric should treat stuff and thing classes in a uniform way, capturing all aspects of the task.

\textbf{Interpretability.} We seek a metric with identifiable meaning that facilitates communication and understanding.

\textbf{Simplicity.} In addition, the metric should be simple to define and implement. This improves transparency and allows for easy reimplementation. Related to this, the metric should be efficient to compute to enable rapid evaluation.

Guided by these principles, we propose a new \emph{panoptic quality} (PQ) metric. PQ measures the quality of a predicted panoptic segmentation relative to the ground truth. It involves two steps: (1) segment matching and (2) PQ computation given the matches. We describe each step next then return to a comparison to existing metrics.

\subsection{Segment Matching}\label{sec:matching}

We specify that a predicted segment and a ground truth segment can match only if their intersection over union (IoU) is strictly greater than 0.5. This requirement, together with the non-overlapping property of a panoptic segmentation, gives a \emph{unique matching}: there can be at most one predicted segment matched with each ground truth segment.
\begin{theorem} \label{thm:overlap}
Given a predicted and ground truth panoptic segmentation of an image, each ground truth segment can have at most one corresponding predicted segment with IoU strictly greater than 0.5 and vice verse.
\end{theorem}
\begin{proof}
Let $g$ be a ground truth segment and $p_1$ and $p_2$ be two predicted segments. By definition, $p_1 \cap p_2 = \emptyset$ (they do not overlap). Since $|p_i \cup g|\ge|g|$, we get the following:
\begin{align} \nonumber
\IoU(p_i, g) = \frac{|p_i \cap g|}{|p_i \cup g|} \le \frac{|p_i \cap g|}{|g|} \quad \text{for} \, i \in \{1,2\} \,.
\end{align}
Summing over $i$, and since $|p_1 \cap g| + |p_2 \cap g| \le |g|$ due to the fact that $p_1 \cap p_2 = \emptyset$, we get:
\begin{align} \nonumber
\IoU(p_1, g) + \IoU(p_2, g) \le \frac{|p_1 \cap g| + |p_2 \cap g|}{|g|} \le 1 \,.
\end{align}
Therefore, if $\IoU(p_1, g) > 0.5$, then $\IoU(p_2, g)$ has to be smaller than 0.5. Reversing the role of $p$ and $g$ can be used to prove that only one ground truth segment can have IoU with a predicted segment strictly greater than 0.5.
\end{proof}

The requirement that matches must have IoU greater than 0.5, which in turn yields the unique matching theorem, achieves two of our desired properties. First, it is \emph{simple} and efficient as correspondences are unique and trivial to obtain. Second, it is \emph{interpretable} and easy to understand (and does not require solving a complex matching problem as is commonly the case for these types of metrics \cite{hariharan2014simultaneous,yang2012layered}).

Note that due to the uniqueness property, for IoU $>$ 0.5, any reasonable matching strategy (including greedy and optimal) will yield an identical matching. For smaller IoU other matching techniques would be required; however, in the experiments we will show that lower thresholds are unnecessary as matches with IoU $\le$ 0.5 are rare in practice.

\subsection{PQ Computation}

\begin{figure}[t]
\centering
\includegraphics[width=1\linewidth]{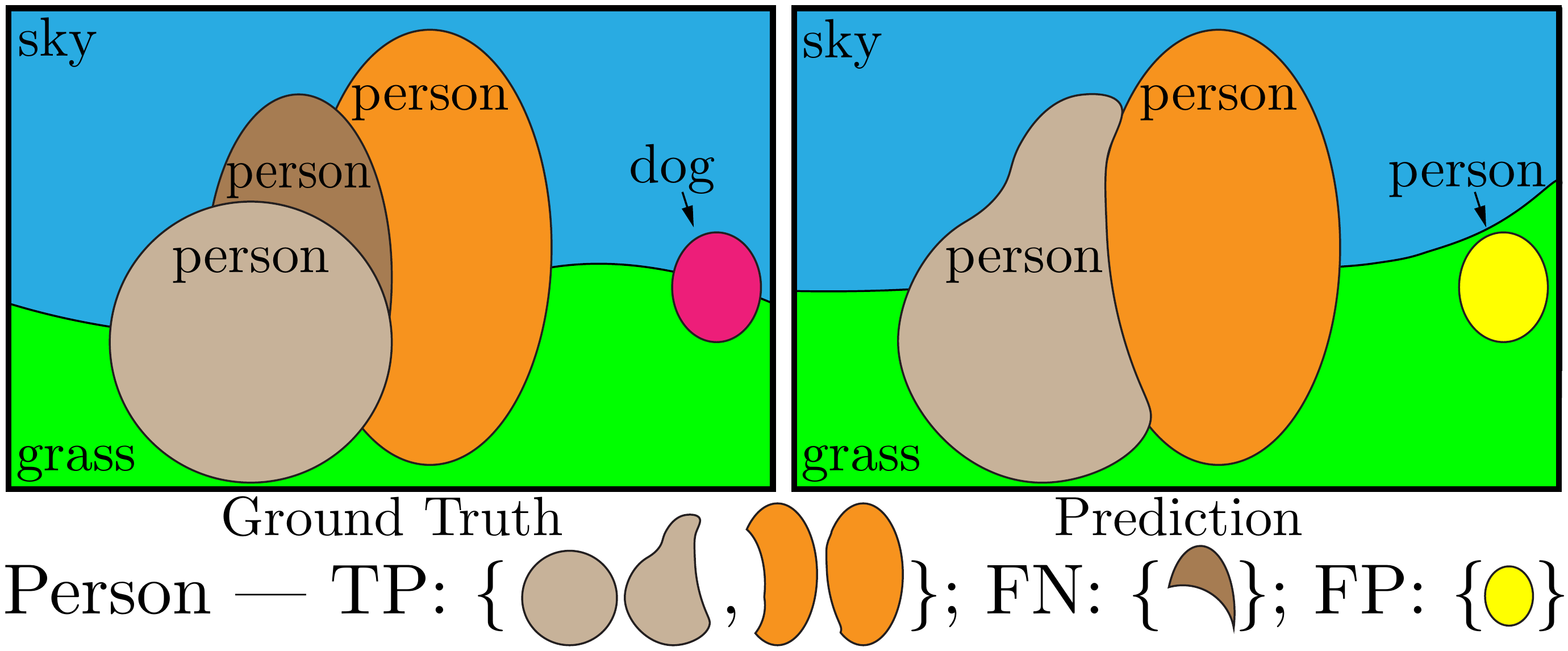}
\caption{Toy illustration of ground truth and predicted panoptic segmentations of an image. Pairs of segments of the same color have IoU larger than 0.5 and are therefore matched. We show how the segments for the \emph{person} class are partitioned into true positives $\TP$, false negatives $\FN$, and false positives $\FP$.}
\label{fig:psq-overview}\vspace{-1mm}
\end{figure}

We calculate PQ for each class independently and average over classes. This makes PQ insensitive to class imbalance. For each class, the unique matching splits the predicted and ground truth segments into three sets: true positives ($\TP$), false positives ($\FP$), and false negatives ($\FN$), representing matched pairs of segments, unmatched predicted segments, and unmatched ground truth segments, respectively. An example is illustrated in Figure~\ref{fig:psq-overview}. Given these three sets, PQ is defined as:
 \eqnsm{psq}{\text{PQ} = \frac{\sum_{(p, g) \in \TP} \text{IoU}(p, g)}{|\TP| + \frac{1}{2}|\FP| + \frac{1}{2}|\FN|} \,.}
PQ is intuitive after inspection: $\frac{1}{|\TP|} \sum_{(p, g) \in \TP} \IoU(p, g)$ is simply the average IoU of matched segments, while $\frac{1}{2} |\FP| + \frac{1}{2} |\FN|$ is added to the denominator to penalize segments without matches. Note that all segments receive equal importance regardless of their area. Furthermore, if we multiply and divide PQ by the size of the $\TP$ set, then PQ can be seen as the multiplication of a \emph{segmentation quality} (SQ) term and a \emph{recognition quality} (RQ) term:
 \eqnsm{psq-seg-det}{\small{\text{PQ}} = \underbrace{\frac{\sum_{(p, g) \in \TP} \text{IoU}(p, g)}{\vphantom{\frac{1}{2}}|\TP|}}_{\text{segmentation quality (SQ)}} \times \underbrace{\frac{|\TP|}{|\TP| + \frac{1}{2} |\FP| + \frac{1}{2} |\FN|}}_{\text{recognition quality (RQ) }} \,.}
Written this way, RQ is the familiar $F_1$ score \cite{van1979information} widely used for quality estimation in detection settings \cite{martin2004learning}. SQ is simply the average IoU of matched segments. We find the decomposition of PQ = SQ $\times$ RQ to provide insight for analysis. We note, however, that the two values are not independent since SQ is measured only over matched segments.

Our definition of PQ achieves our desiderata. It measures performance of all classes in a uniform way using a simple and interpretable formula. We conclude by discussing how we handle void regions and groups of instances \cite{lin2014coco}.

\textbf{Void labels.} There are two sources of void labels in the ground truth: (a) out of class pixels and (b) ambiguous or unknown pixels. As often we cannot differentiate these two cases, we don't evaluate predictions for void pixels. Specifically: (1) during matching, all pixels in a predicted segment that are labeled as void in the ground truth are removed from the prediction and do not affect IoU computation, and (2) after matching, unmatched predicted segments that contain a fraction of void pixels over the matching threshold are removed and do not count as false positives. Finally, outputs may also contain void pixels; these do not affect evaluation.

\textbf{Group labels.} A common annotation practice \cite{Cordts2016Cityscapes, lin2014coco} is to use a group label instead of instance ids for adjacent instances of the same semantic class if accurate delineation of each instance is difficult. For computing PQ: (1) during matching, group regions are not used, and (2) after matching, unmatched predicted segments that contain a fraction of pixels from a group of the same class over the matching threshold are removed and do not count as false positives.

\subsection{Comparison to Existing Metrics}

We conclude by comparing PQ to existing metrics for semantic and instance segmentation.

\paragraph{Semantic segmentation metrics.} Common metrics for semantic segmentation include pixel accuracy, mean accuracy, and IoU \cite{long2015fully}. These metrics are computed based only on pixel outputs/labels and completely ignore object-level labels. For example, IoU is the ratio between correctly predicted pixels and total number of pixels in either the prediction or ground truth for each class. As these metrics ignore instance labels, they are not well suited for evaluating thing classes. Finally, please note that IoU for semantic segmentation is distinct from our segmentation quality (SQ), which is computed as the average IoU over \emph{matched segments}.

\paragraph{Instance segmentation metrics.} The standard metric for instance segmentation is Average Precision (AP) \cite{lin2014coco,hariharan2014simultaneous}. AP requires each object segment to have a confidence score to estimate a precision/recall curve. Note that while confidence scores are quite natural for object detection, they are not used for semantic segmentation. Hence, AP cannot be used for measuring the output of semantic segmentation, or likewise of PS (see also the discussion of confidences in \S\ref{sec:task}).

\paragraph{Panoptic quality.} PQ treats all classes (stuff and things) in a uniform way. We note that while decomposing PQ into SQ and RQ is helpful with interpreting results, PQ is \emph{not} a combination of semantic and instance segmentation metrics. Rather, SQ and RQ are computed for every class (stuff and things), and measure segmentation and recognition quality, respectively. PQ thus unifies evaluation over all classes. We support this claim with rigorous experimental evaluation of PQ in \S\ref{sec:machines}, including comparisons to IoU and AP for semantic and instance segmentation, respectively.

\section{Panoptic Segmentation Datasets}

To our knowledge only three public datasets have both dense semantic and instance segmentation annotations: Cityscapes \cite{Cordts2016Cityscapes}, ADE20k \cite{zhou2017ade20k}, and Mapillary Vistas \cite{neuhold2017mapillary}. We use all three datasets for panoptic segmentation. In addition, in the future we will extend our analysis to COCO \cite{lin2014coco} on which stuff has been recently annotated \cite{caesar2016coco}\footnote{COCO instance segmentations contain overlaps. We collected depth ordering for all pairs of overlapping instances in COCO to resolve these overlaps: {\fontsize{7pt}{1em}\url{http://cocodataset.org/\#panoptic-2018}}.}.

\textbf{Cityscapes} \cite{Cordts2016Cityscapes} has 5000 images (2975 train, 500 val, and 1525 test) of ego-centric driving scenarios in urban settings. It has dense pixel annotations (97\% coverage) of 19 classes among which 8 have instance-level segmentations.

\textbf{ADE20k} \cite{zhou2017ade20k} has over 25k images (20k train, 2k val, 3k test) that are densely annotated with an open-dictionary label set. For the 2017 Places Challenge\footnote{\fontsize{7pt}{1em}\url{http://placeschallenge.csail.mit.edu}}, 100 thing and 50 stuff classes that cover 89\% of all pixels are selected. We use this closed vocabulary in our study.

\textbf{Mapillary Vistas} \cite{neuhold2017mapillary} has 25k street-view images (18k train, 2k val, 5k test) in a wide range of resolutions. The `research edition' of the dataset is densely annotated (98\% pixel coverage) with 28 stuff and 37 thing classes.

\section{Human Consistency Study}\label{sec:human}

One advantage of panoptic segmentation is that it enables measuring human annotation consistency. Aside from this being interesting as an end in itself, human consistency studies allow us to understand the task in detail, including details of our proposed metric and breakdowns of human consistency along various axes. This gives us insight into intrinsic challenges posed by the task without biasing our analysis by algorithmic choices. Furthermore, human studies help ground machine performance (discussed in \S\ref{sec:machines}) and allow us to calibrate our understanding of the task.

\begin{figure}
\begin{subfigure}{1.0\linewidth}
\centering
 \includegraphics[width=0.32\linewidth]{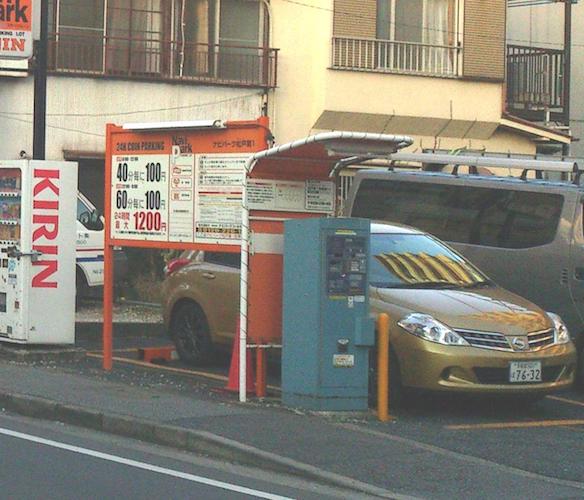}
 \includegraphics[width=0.32\linewidth]{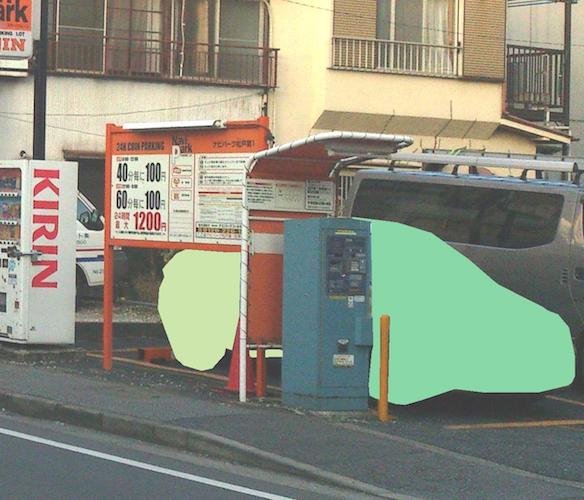}
 \includegraphics[width=0.32\linewidth]{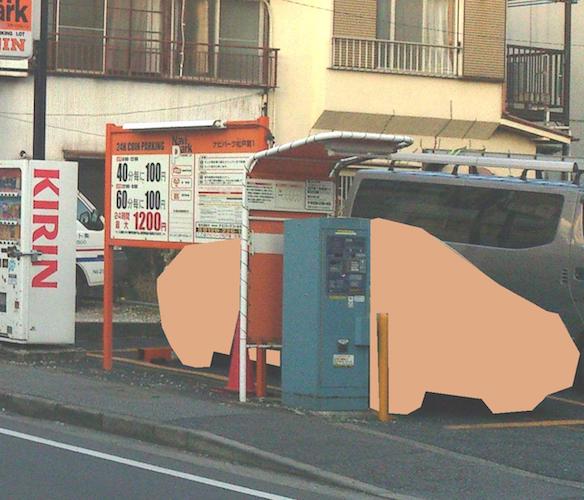}
\end{subfigure}\\[1pt]
\begin{subfigure}{1.0\linewidth}
\centering
 \includegraphics[width=0.32\linewidth]{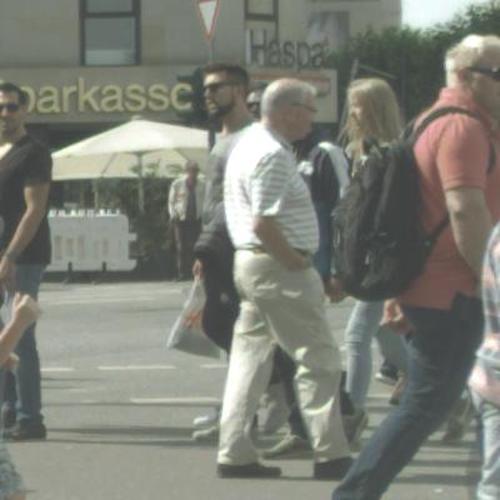}
 \includegraphics[width=0.32\linewidth]{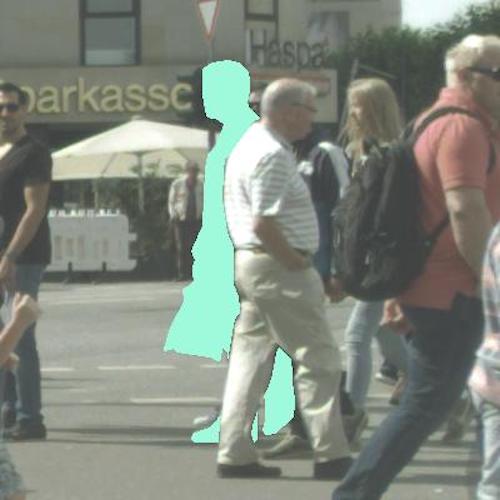}
 \includegraphics[width=0.32\linewidth]{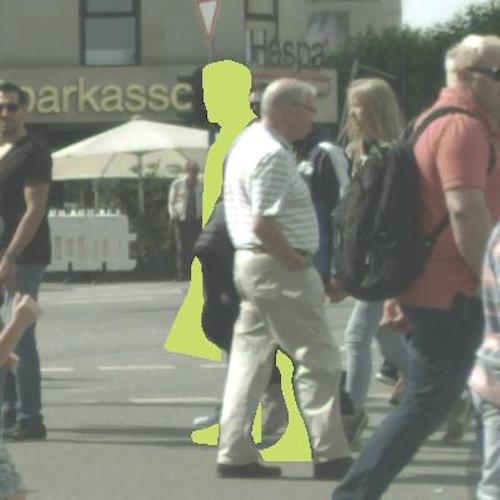}
\end{subfigure}\vspace{-2mm}
\caption{\textbf{Segmentation flaws.} Images are zoomed and cropped. Top row (Vistas image): both annotators identify the object as a car, however, one splits the car into two cars. Bottom row (Cityscapes image): the segmentation is genuinely ambiguous.}\vspace{-1mm}
\label{fig:segm-flaws}
\end{figure}

\begin{table}
\resizebox{1\linewidth}{!}{\tablestyle{4pt}{1.1}
\begin{tabular}{@{}l|ccc|ccc|ccc@{}}
 & PQ & PQ\stuff & PQ\things\ & SQ & SQ\stuff & SQ\things\ & RQ & RQ\stuff & RQ\things\\\shline
 Cityscapes & 69.7 & 71.3 & 67.4 & 84.2 & 84.4 & 83.9 & 82.1 & 83.4 & 80.2\\
 ADE20k     & 67.1 & 70.3 & 65.9 & 85.8 & 85.5 & 85.9 & 78.0 & 82.4 & 76.4\\
 Vistas     & 57.5 & 62.6 & 53.4 & 79.5 & 81.6 & 77.9 & 71.4 & 76.0 & 67.7\\
\end{tabular}}
\caption{\textbf{Human consistency for stuff \vs things}. Panoptic, segmentation, and recognition quality (PQ, SQ, RQ) averaged over classes (PQ=SQ$\times$RQ per class) are reported as percentages. Perhaps surprisingly, we find that human consistency on each dataset is relatively similar for both stuff and things.}
\label{tab:stuff-things}
\end{table}

\paragraph{Human annotations.} To enable human consistency analysis, dataset creators graciously supplied us with 30 doubly annotated images for Cityscapes, 64 for ADE20k, and 46 for Vistas. For Cityscapes and Vistas, the images are annotated independently by different annotators. ADE20k is annotated by a single well-trained annotator who labeled the same set of images with a gap of six months. To measure panoptic quality (PQ) for human annotators, we treat one annotation for each image as ground truth and the other as the prediction. Note that the PQ is symmetric \wrt the ground truth and prediction, so order is unimportant.

\paragraph{Human consistency.} First, Table~\ref{tab:stuff-things} shows human consistency on each dataset, along with the decomposition of PQ into segmentation quality (SQ) and recognition quality (RQ). As expected, humans are not perfect at this task, which is consistent with studies of annotation quality from \cite{Cordts2016Cityscapes,zhou2017ade20k,neuhold2017mapillary}. Visualizations of human segmentation and classification errors are shown in Figures \ref{fig:segm-flaws} and \ref{fig:class-flaws}, respectively.

We note that Table~\ref{tab:stuff-things} establishes a measure of annotator agreement on each dataset, \emph{not} an upper bound on human consistency. We further emphasize that numbers are not comparable across datasets and should not be used to assess dataset quality. The number of classes, percent of annotated pixels, and scene complexity vary across datasets, each of which significantly impacts annotation difficulty.

\paragraph{Stuff \vs things.} PS requires segmentation of both stuff and things. In Table~\ref{tab:stuff-things} we also show PQ\stuff and PQ\things which is the PQ averaged over stuff classes and thing classes, respectively. For Cityscapes and ADE20k human consistency for stuff and things are close, on Vistas the gap is a bit larger. Overall, this implies stuff and things have similar difficulty, although thing classes are somewhat harder. In Figure~\ref{fig:psq-per-class} we show PQ for every class in each dataset, sorted by PQ. Observe that stuff and things classes distribute fairly evenly. This implies that the proposed metric strikes a good balance and, indeed, is successful at unifying the stuff and things segmentation tasks without either dominating the error.

\begin{figure}
\begin{subfigure}{1.0\linewidth}
\centering
 \includegraphics[width=0.32\linewidth]{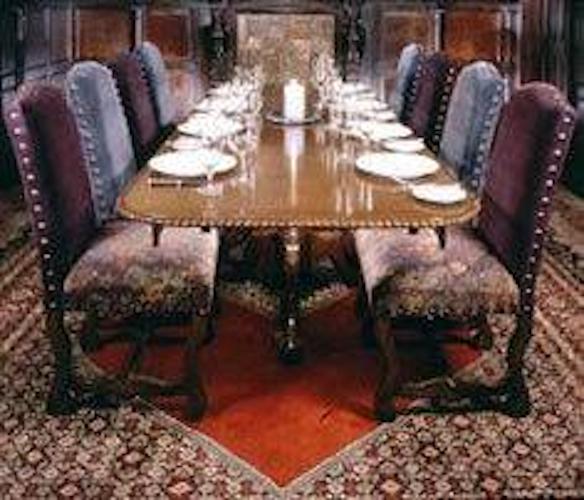}
 \includegraphics[width=0.32\linewidth]{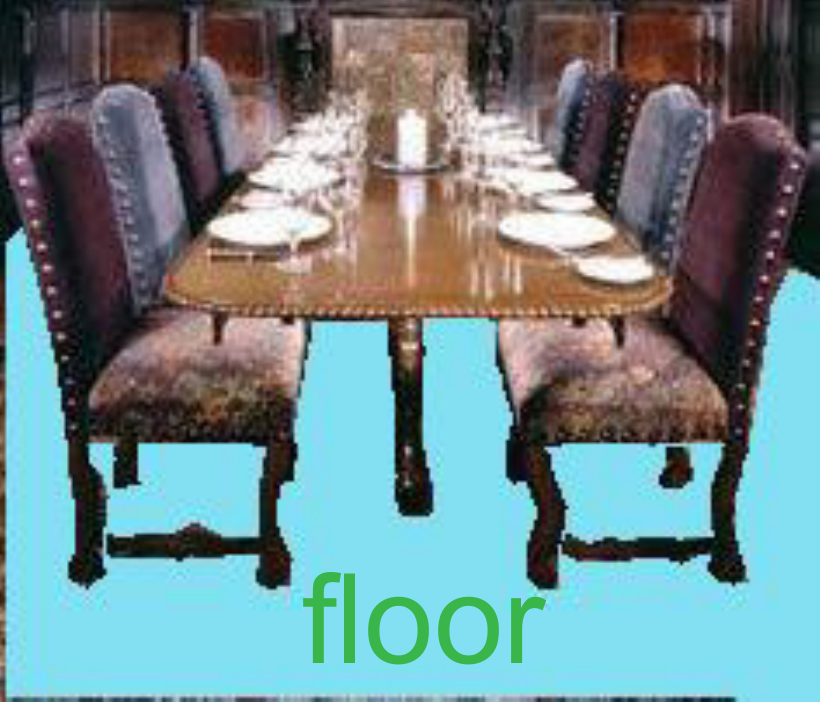}
 \includegraphics[width=0.32\linewidth]{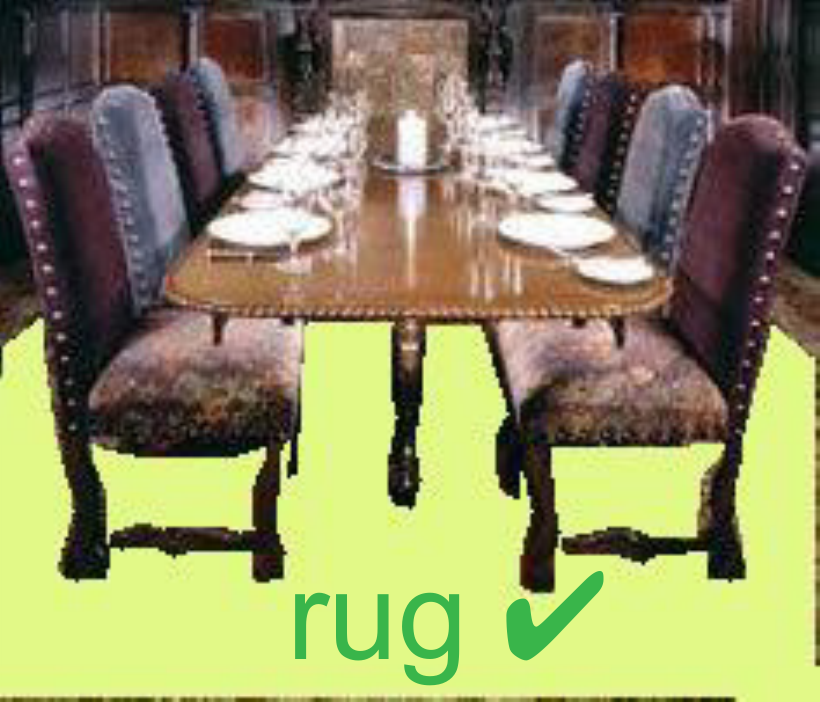}
\end{subfigure}\\[1pt]
\begin{subfigure}{1.0\linewidth}
\centering
 \includegraphics[width=0.32\linewidth]{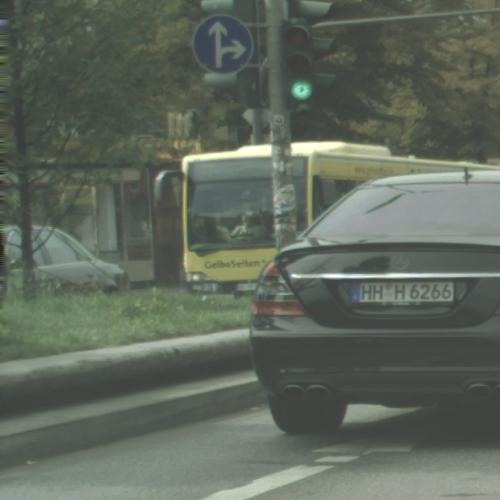}
 \includegraphics[width=0.32\linewidth]{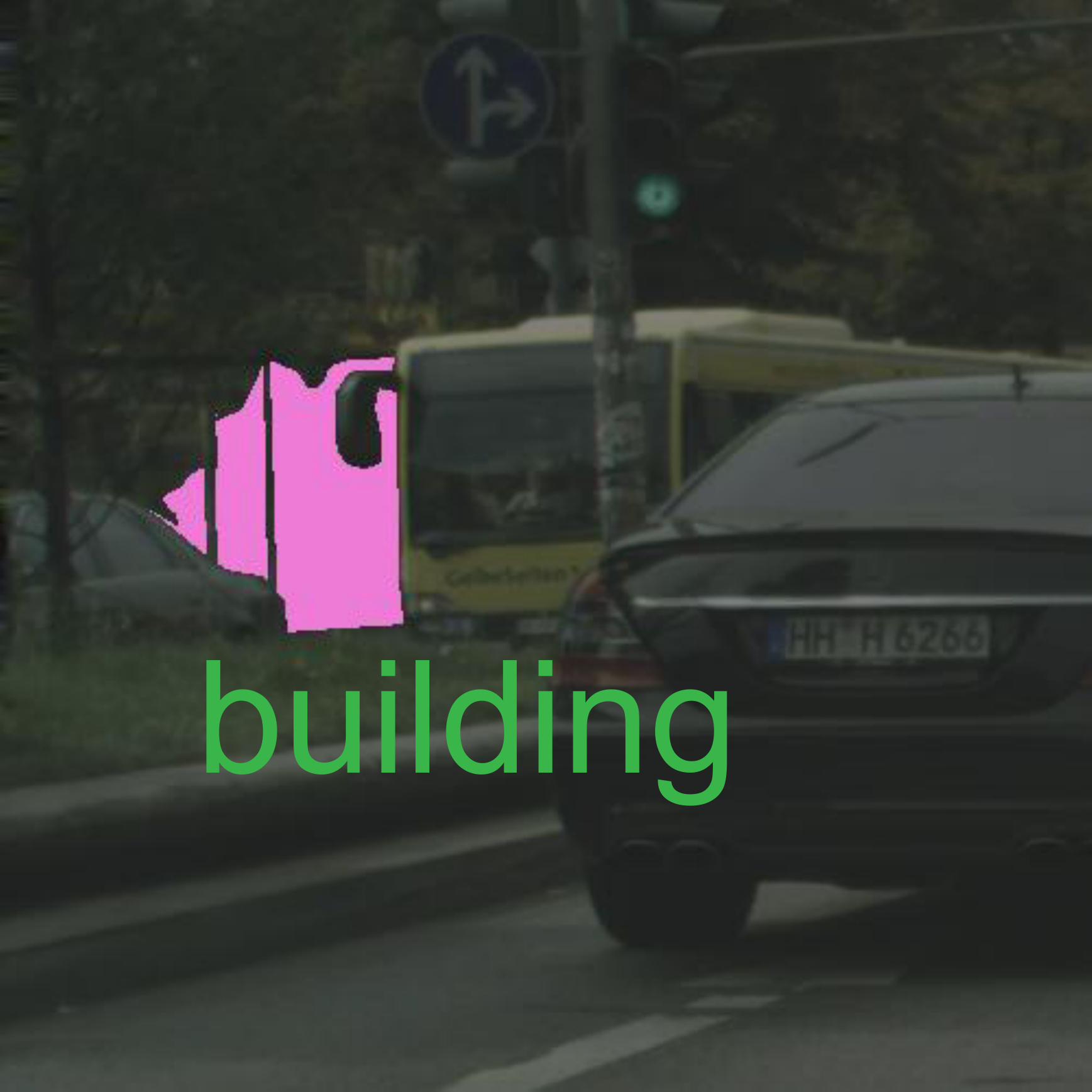}
 \includegraphics[width=0.32\linewidth]{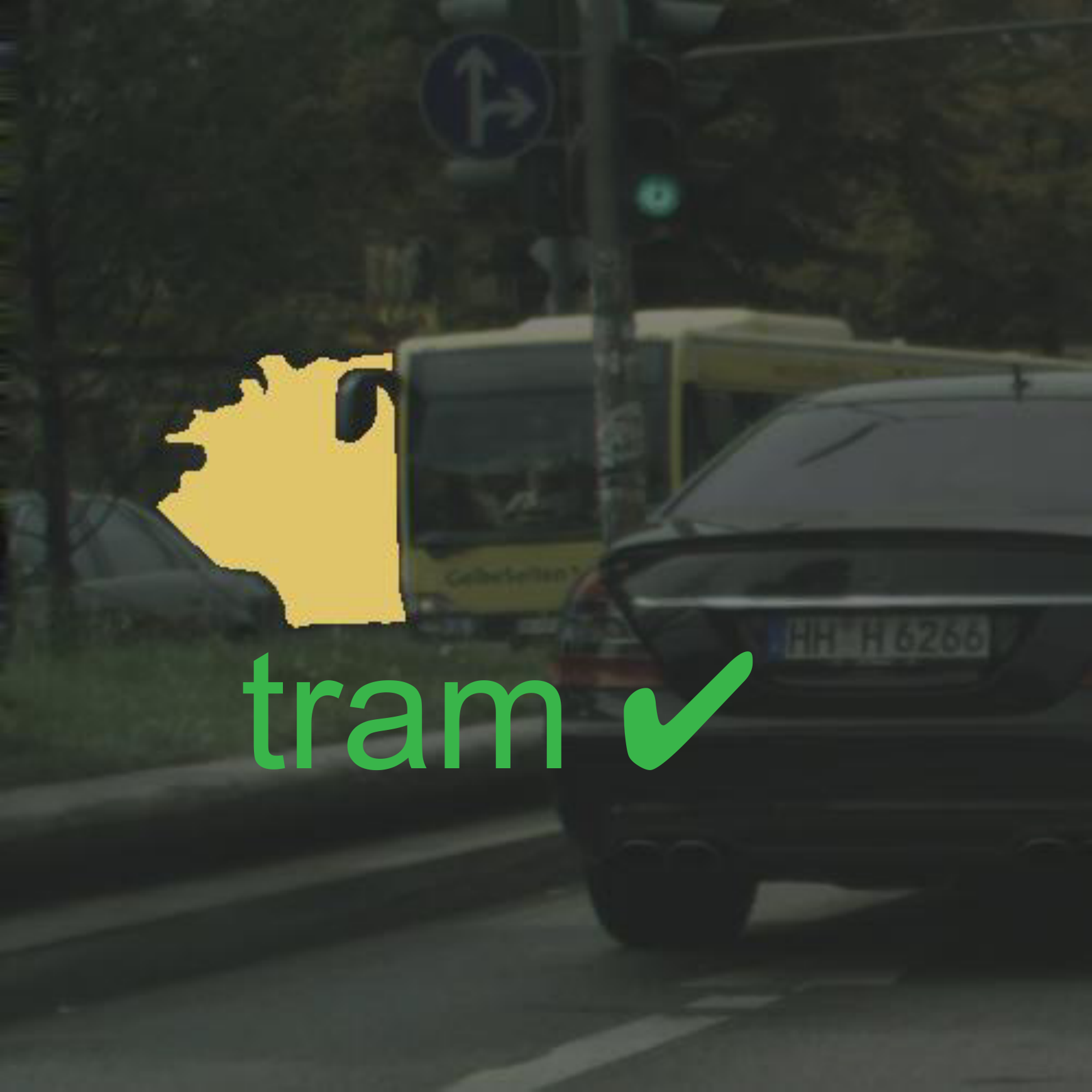}
\end{subfigure}\vspace{-2mm}
\caption{\textbf{Classification flaws.} Images are zoomed and cropped. Top row (ADE20k image): simple misclassification. Bottom row (Cityscapes image): the scene is extremely difficult, tram is the correct class for the segment. Many errors are difficult to resolve.}\vspace{-1mm}
\label{fig:class-flaws}
\end{figure}

\begin{table}
\resizebox{1\linewidth}{!}{\tablestyle{4pt}{1.1}
\begin{tabular}{@{}l|ccc|ccc|ccc@{}}
 & PQ\tss{S} & PQ\tss{M} & PQ\tss{L} & SQ\tss{S} & SQ\tss{M} & SQ\tss{L} & RQ\tss{S} & RQ\tss{M} & RQ\tss{L}\\\shline
 Cityscapes & 35.1 & 62.3 & 84.8 & 67.8 & 81.0 & 89.9 & 51.5 & 76.5 & 94.1 \\
 ADE20k     & 49.9 & 69.4 & 79.0 & 78.0 & 84.0 & 87.8 & 64.2 & 82.5 & 89.8 \\
 Vistas     & 35.6 & 47.7 & 69.4 & 70.1 & 76.6 & 83.1 & 51.5 & 62.3 & 82.6
\end{tabular}}
\caption{\textbf{Human consistency \vs scale}, for small (S), medium (M) and large (L) objects. Scale plays a large role in determining human consistency for panoptic segmentation. On large objects both SQ and RQ are above 80 on all datasets, while for small objects RQ drops precipitously. SQ for small objects is quite reasonable.}
\label{tab:scale}
\end{table}

\paragraph{Small \vs large objects.} To analyze how PQ varies with object size we partition the datasets into small (S), medium (M), and large (L) objects by considering the smallest $25\%$, middle $50\%$, and largest $25\%$ of objects in each dataset, respectively. In Table~\ref{tab:scale}, we see that for large objects human consistency for all datasets is quite good. For small objects, RQ drops significantly implying human annotators often have a hard time finding small objects. However, if a small object is found, it is segmented relatively well.

\paragraph{IoU threshold.} By enforcing an overlap greater than 0.5 IoU, we are given a unique matching by Theorem ~\ref{thm:overlap}. However, is the 0.5 threshold reasonable? An alternate strategy is to use no threshold and perform the matching by solving a maximum weighted bipartite matching problem \cite{west2001introduction}. The optimization will return a matching that maximizes the sum of IoUs of the matched segments. We perform the matching using this optimization and plot the cumulative density functions of the match overlaps in Figure~\ref{fig:matching_cdf}. Less than 16\% of the matches have IoU overlap less than 0.5, indicating that relaxing the threshold should have minor effect.

To verify this intuition, in Figure~\ref{fig:iou-threshold} we show PQ computed for different IoU thresholds. Notably, the difference in PQ for IoU of 0.25 and 0.5 is relatively small, especially compared to the gap between IoU of 0.5 and 0.75, where the change in PQ is larger. Furthermore, many matches at lower IoU are false matches. Therefore, given that the matching for IoU of 0.5 is not only unique, but also simple and intuitive, we believe that the default choice of 0.5 is reasonable.

\begin{figure}[t]
\includegraphics[width=1.0\linewidth]{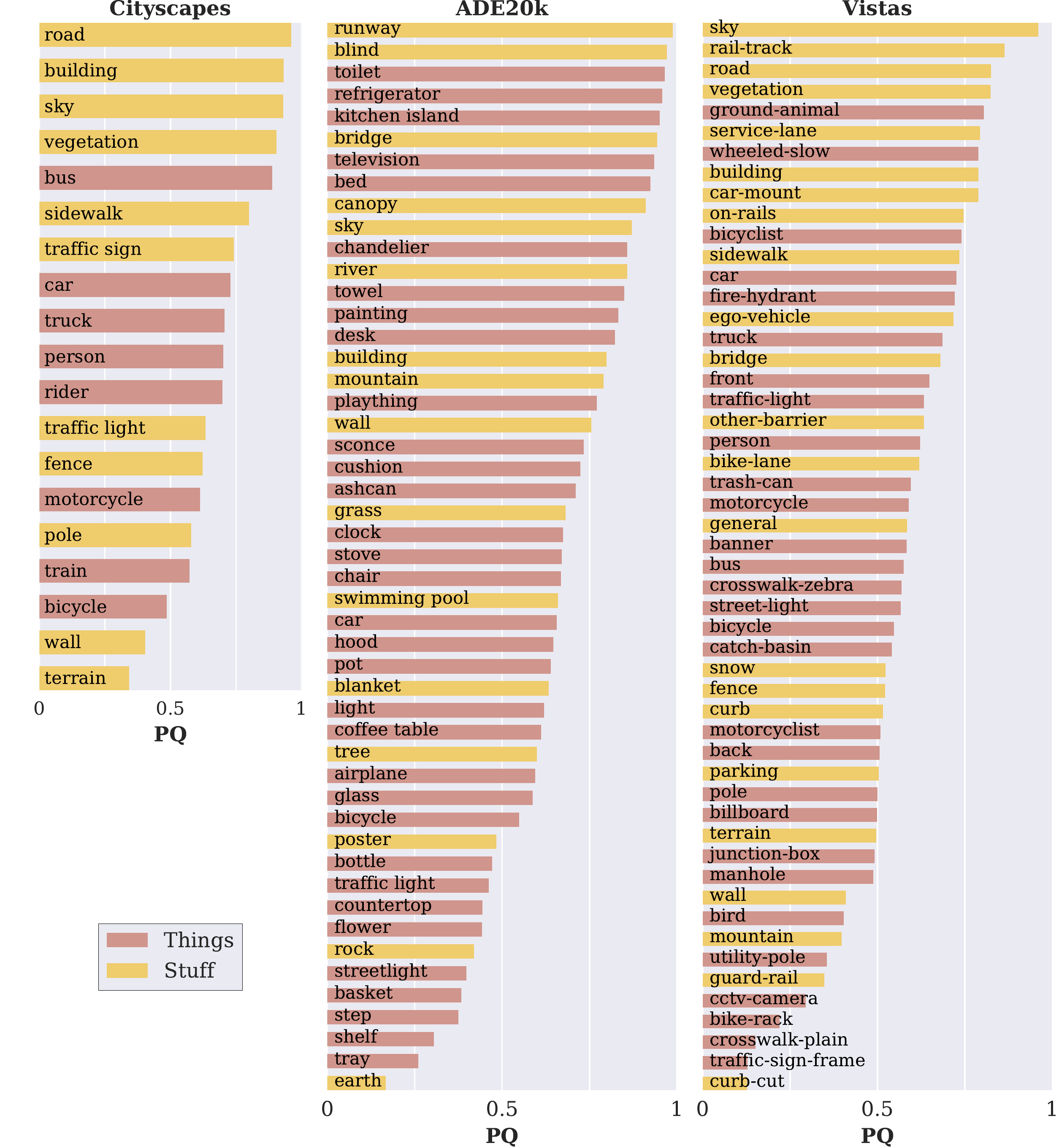}
\caption{\textbf{Per-Class Human consistency, sorted by PQ}. Thing classes are shown in red, stuff classes in orange (for ADE20k every other class is shown, classes without matches in the dual-annotated tests sets are omitted). Things and stuff are distributed fairly evenly, implying PQ balances their performance.}
\label{fig:psq-per-class}
\end{figure}

\begin{figure}[t]
\includegraphics[width=1.0\linewidth]{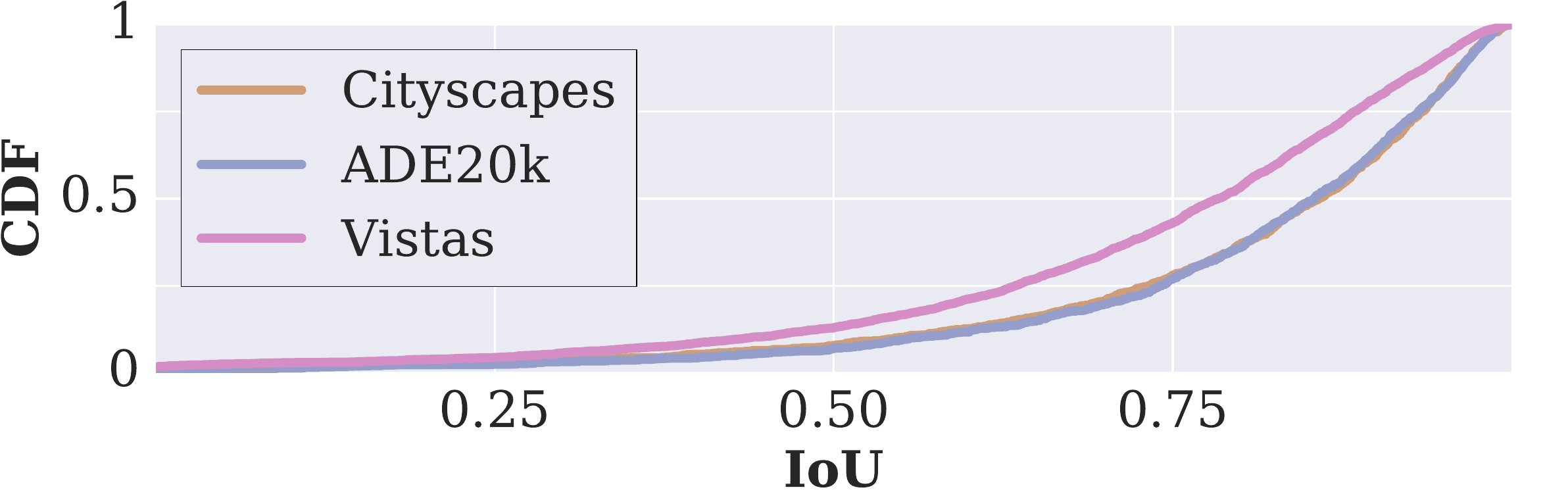}
\caption{\textbf{Cumulative density functions of overlaps} for matched segments in three datasets when matches are computed by solving a maximum weighted bipartite matching problem \cite{west2001introduction}. After matching, less than 16\% of matched objects have IoU below 0.5.}
\label{fig:matching_cdf}
\end{figure}

\begin{figure}[t]
\includegraphics[width=1.0\linewidth]{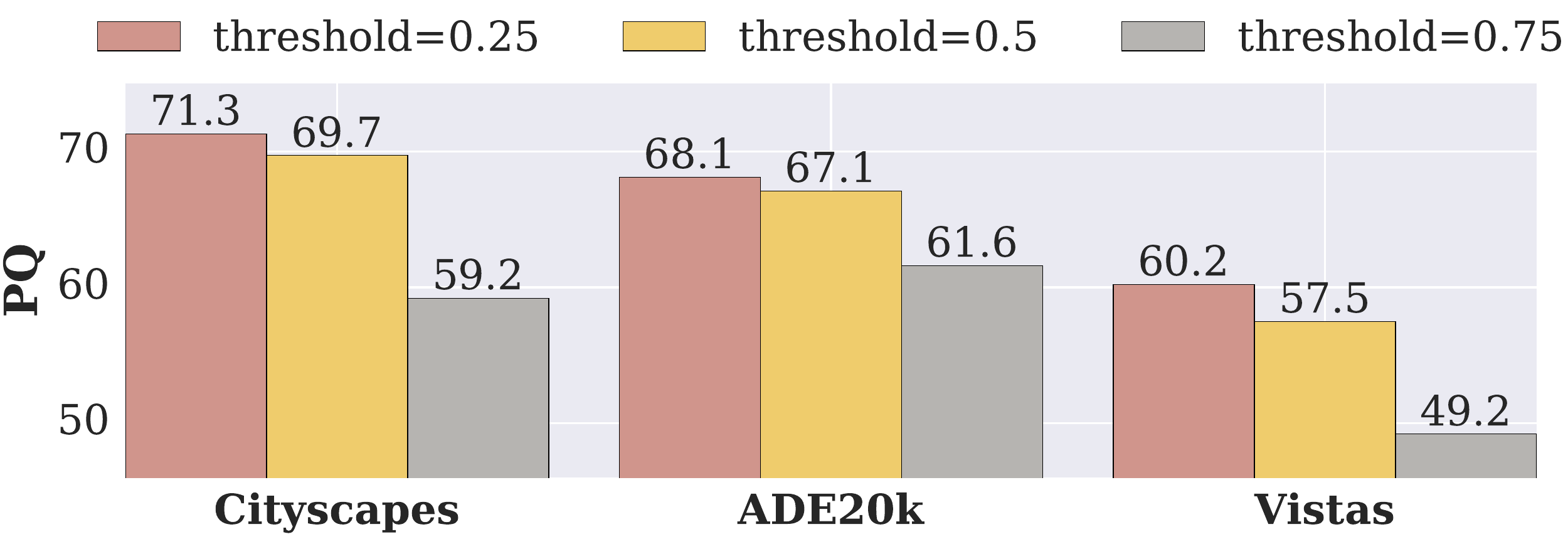}
\caption{\textbf{Human consistency for different IoU thresholds.} The difference in PQ using a matching threshold of 0.25 \vs 0.5 is relatively small. For IoU of 0.25 matching is obtained by solving a maximum weighted bipartite matching problem. For a threshold greater than 0.5 the matching is unique and much easier to obtain.}
\label{fig:iou-threshold}
\end{figure}

\paragraph{SQ \vs RQ balance.} Our RQ definition is equivalent to the $F_1$ score. However, other choices are possible. Inspired by the generalized $F_\beta$ score \cite{van1979information}, we can introduce a parameter $\alpha$ that enables tuning the penalty for recognition errors:
 \eqnsm{alpha}{\text{RQ}^\alpha = \frac{|\TP|}{|\TP| + \alpha |\FP| + \alpha |\FN|} \,.}
By default $\alpha$ is 0.5. Lowering $\alpha$ reduces the penalty of unmatched segments and thus increases RQ (SQ is not affected). Since PQ=SQ$\times$RQ, this changes the relative effect of PS \vs RQ on the final PQ metric. In Figure~\ref{fig:alpha-balance} we show SQ and RQ for various $\alpha$. The default $\alpha$ strikes a good balance between SQ and RQ. In principle, altering $\alpha$ can be used to balance the influence of segmentation and recognition errors on the final metric. In a similar spirit, one could also add a parameter $\beta$ to balance influence of FPs \vs FNs.

\begin{figure}
\includegraphics[width=1.0\linewidth]{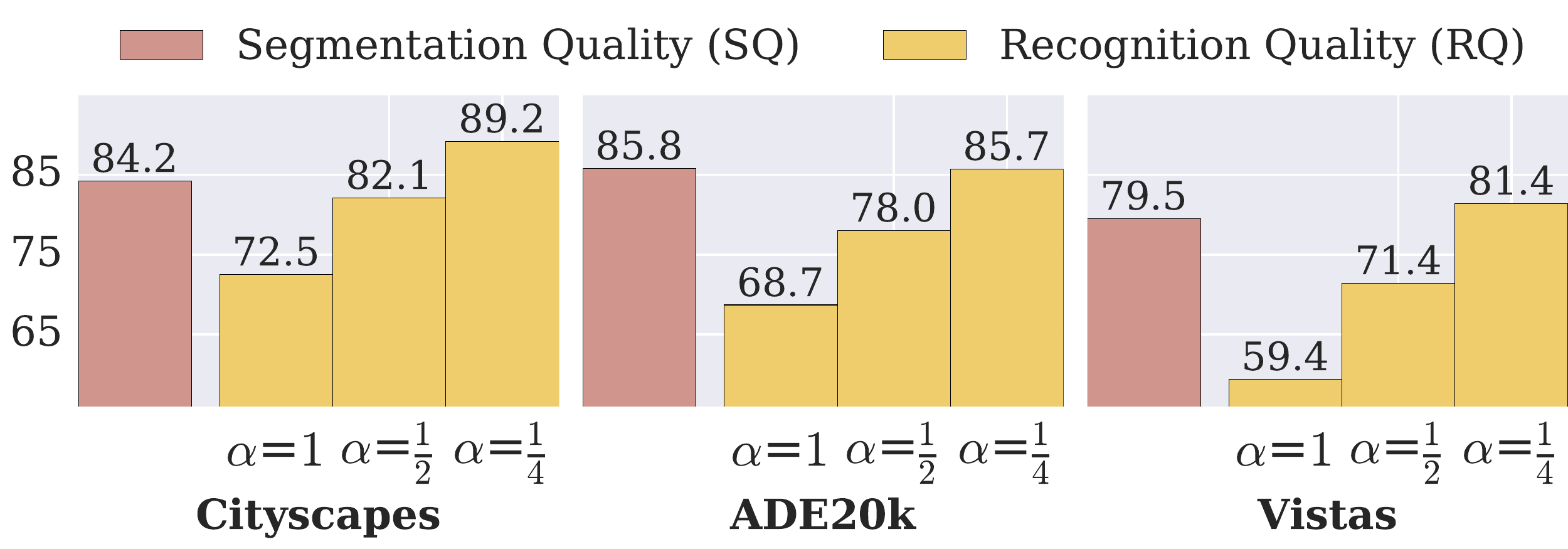}
\caption{\textbf{SQ \vs RQ} for different $\alpha$, see (\ref{eq:alpha}). Lowering $\alpha$ reduces the penalty of unmatched segments and thus increases the reported RQ (SQ is not affected). We use $\alpha$ of 0.5 throughout but by tuning $\alpha$ one can balance the influence of SQ and RQ in the final metric.}
\label{fig:alpha-balance}
\end{figure}

\section{Machine Performance Baselines}\label{sec:machines}

We now present simple machine baselines for panoptic segmentation. We are interested in three questions: (1) How do heuristic combinations of top-performing instance and semantic segmentation systems perform on panoptic segmentation? (2) How does PQ compare to existing metrics like AP and IoU? (3) How do the machine results compare to the human results that we presented previously?

\paragraph{Algorithms and data.} We want to understand panoptic segmentation in terms of existing well-established methods. Therefore, we create a basic PS system by applying reasonable heuristics (described shortly) to the output of existing top instance and semantic segmentation systems.

We obtained algorithm output for three datasets. For \textit{Cityscapes}, we use the val set output generated by the current leading algorithms (PSPNet \cite{zhao2017pspnet} and Mask R-CNN \cite{he2017mask} for semantic and instance segmentation, respectively). For \textit{ADE20k}, we received output for the winners of both the semantic \cite{places_sem_best,places_sem_second} and instance \cite{places_inst_best,places_inst_second} segmentation tracks on a 1k subset of test images from the 2017 Places Challenge. For \textit{Vistas}, which is used for the LSUN'17 Segmentation Challenge, the organizers provide us with 1k test images and results from the winning entries for the instance and semantic segmentation tracks \cite{mapillary_inst,mapillary_sem}.

Using this data, we start by analyzing PQ for the instance and semantic segmentation tasks separately, and then examine the full panoptic segmentation task. Note that our `baselines' are very powerful and that simpler baselines may be more reasonable for fair comparison in papers on PS.

\begin{table}
\tablestyle{4pt}{1.05}
\begin{tabular}{@{}l|cc|ccc@{}}
 {\bf Cityscapes}
 & AP & AP\tss{NO} & PQ\things & SQ\things & RQ\things\\\shline
Mask R-CNN+COCO \cite{he2017mask}
 & {\bf 36.4} & {\bf 33.1} & {\bf 54.0} & {\bf 79.4} & {\bf 67.8}\\
Mask R-CNN \cite{he2017mask}
 & 31.5 & 28.0 & 49.6 & 78.7 & 63.0\\
\multicolumn{6}{c}{\vspace{-0.8em}}\\
{\bf ADE20k} & AP & AP\tss{NO} & PQ\things & SQ\things & RQ\things\\\shline
Megvii \cite{places_inst_best}
 & {\bf 30.1} & {\bf 24.8} & {\bf 41.1} & {\bf 81.6} & {\bf 49.6}\\
G-RMI \cite{places_inst_second}
 & 24.6 & 20.6 & 35.3 & 79.3 & 43.2
\end{tabular}
\caption{\textbf{Machine results on instance segmentation} (stuff classes ignored). Non-overlapping predictions are obtained using the proposed heuristic. AP\tss{NO} is AP of the non-overlapping predictions. As expected, removing overlaps harms AP as detectors benefit from predicting multiple overlapping hypotheses. Methods with better AP also have better AP\tss{NO} and likewise improved PQ.}
\label{tab:machines-instance}
\end{table}

\paragraph{Instance segmentation.} Instance segmentation algorithms produce overlapping segments. To measure PQ, we must first resolve these overlaps. To do so we develop a simple non-maximum suppression (NMS)-like procedure. We first sort the predicted segments by their confidence scores and remove instances with low scores. Then, we iterate over sorted instances, starting from the most confident. For each instance we first remove pixels which have been assigned to previous segments, then, if a sufficient fraction of the segment remains, we accept the non-overlapping portion, otherwise we discard the entire segment. All thresholds are selected by grid search to optimize PQ. Results on Cityscapes and ADE20k are shown in Table~\ref{tab:machines-instance} (Vistas is omitted as it only had one entry to the 2017 instance challenge). Most importantly, AP and PQ track closely, and we expect improvements in a detector's AP will also improve its PQ.

\paragraph{Semantic segmentation.} Semantic segmentations have no overlapping segments by design, and therefore we can directly compute PQ. In Table \ref{tab:machines-semantic} we compare mean IoU, a standard metric for this task, to PQ. For Cityscapes, the PQ gap between methods corresponds to the IoU gap. For ADE20k, the gap is much larger. This is because whereas IoU counts correctly predicted pixel, PQ operates at the level of instances. See the Table \ref{tab:machines-semantic} caption for details.

\begin{table}
\tablestyle{6pt}{1.05}
\begin{tabular}{@{}l|c|ccc@{}}
{\bf Cityscapes} & IoU & PQ\stuff & SQ\stuff & RQ\stuff\\\shline
PSPNet multi-scale \cite{zhao2017pspnet}
 & {\bf 80.6} & {\bf 66.6} & {\bf 82.2} & {\bf 79.3}\\
PSPNet single-scale \cite{zhao2017pspnet}
 & 79.6 & 65.2 & 81.6 & 78.0\\
\multicolumn{5}{c}{\vspace{-0.8em}}\\
{\bf ADE20k} & IoU & PQ\stuff & SQ\stuff & RQ\stuff\\\shline
CASIA\_IVA\_JD \cite{places_sem_best}
 & {\bf 32.3} & {\bf 27.4} & {\bf 61.9} & {\bf 33.7}\\
G-RMI \cite{places_sem_second}
 & 30.6 & 19.3 & 58.7 & 24.3
\end{tabular}
\caption{\textbf{Machine results on semantic segmentation} (thing classes ignored). Methods with better mean IoU also show better PQ results. Note that G-RMI has quite low PQ. We found this is because it hallucinates many small patches of classes not present in an image. While this only slightly affects IoU which counts \emph{pixel} errors it severely degrades PQ which counts \emph{instance} errors.}
\label{tab:machines-semantic}
\end{table}

\paragraph{Panoptic segmentation.} To produce algorithm outputs for PS, we start from the non-overlapping instance segments from the NMS-like procedure described previously. Then, we combine those segments with semantic segmentation results by resolving any overlap between thing and stuff classes in favor of the thing class (\ie, a pixel with a thing and stuff label is assigned the thing label and its instance id). This heuristic is imperfect but sufficient as a baseline.

Table~\ref{tab:separate-vs-panoptic} compares PQ\stuff and PQ\things computed on the combined (`panoptic') results to the performance achieved from the separate predictions discussed above. For these results we use the winning entries from each respective competition for both the instance and semantic tasks. Since overlaps are resolved in favor of things, PQ\things is constant while PQ\stuff is slightly lower for the panoptic predictions. Visualizations of panoptic outputs are shown in Figure \ref{fig:examples}.

\paragraph{Human \vs machine panoptic segmentation.} To compare human \vs machine PQ, we use the machine panoptic predictions described above. For human results, we use the dual-annotated images described in \S\ref{sec:human} and use bootstrapping to obtain confidence intervals since these image sets are small. These comparisons are imperfect as they use different test images and are averaged over different classes (some classes without matches in the dual-annotated tests sets are omitted), but they can still give some useful signal.

We present the comparison in Table~\ref{tab:human-vs-machines}. For SQ, machines trail humans only slightly. On the other hand, machine RQ is dramatically lower than human RQ, especially on ADE20k and Vistas. This implies that recognition, \ie, classification, is the main challenge for current methods. Overall, there is a significant gap between human and machine performance. We hope that this gap will inspire future research for the proposed panoptic segmentation task.

\newcommand{\incpr}[1]{\includegraphics[height=.134\linewidth]{pics/panoptic_prediction/#1.jpg}}
\begin{figure*}
\begin{subfigure}{1.0\linewidth}\centering
 \rotatebox{90}{\begin{minipage}{.134\linewidth}\center\small image\end{minipage}}
 \incpr{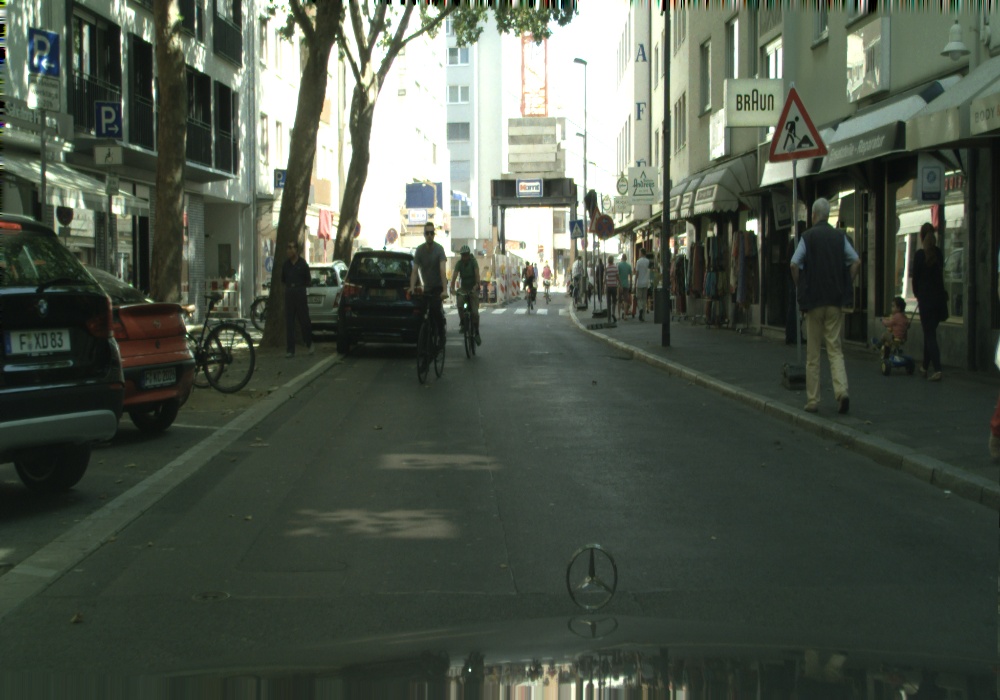}
 \incpr{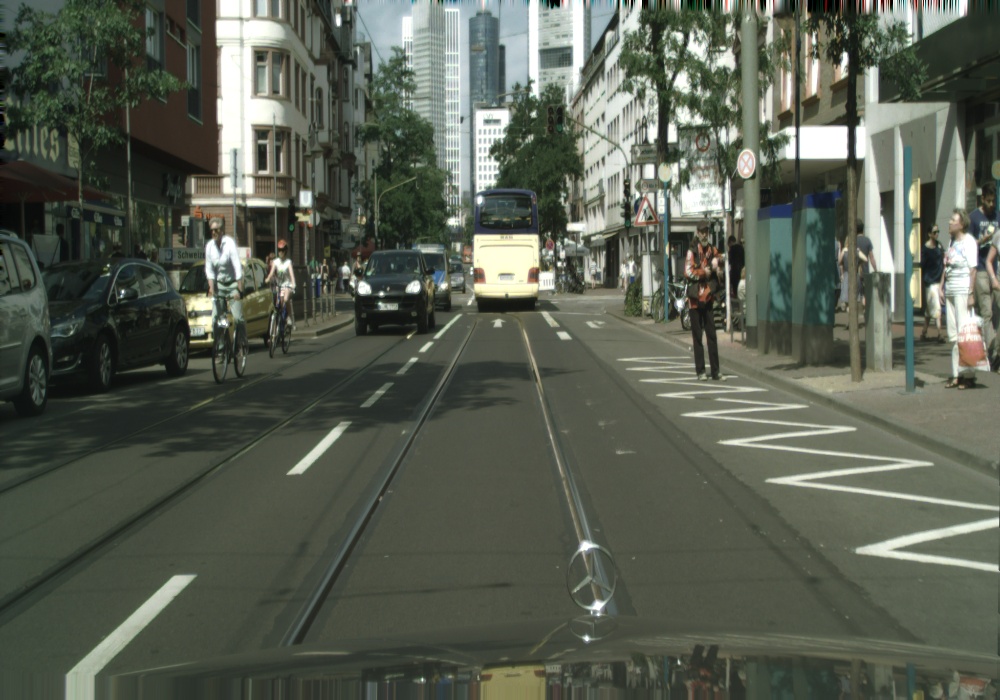}
 \incpr{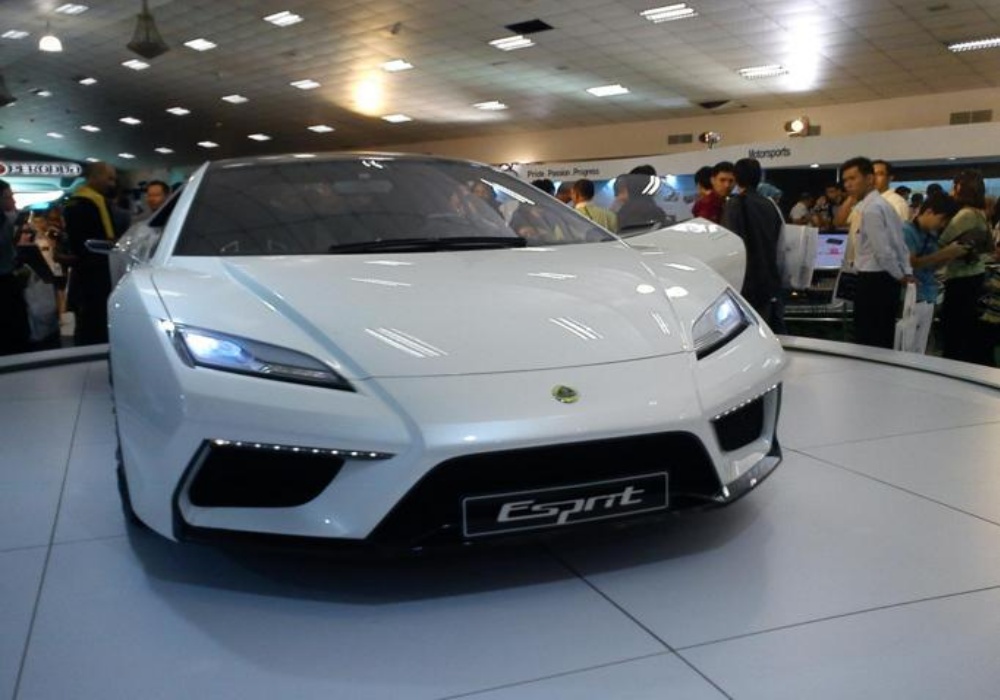}
 \incpr{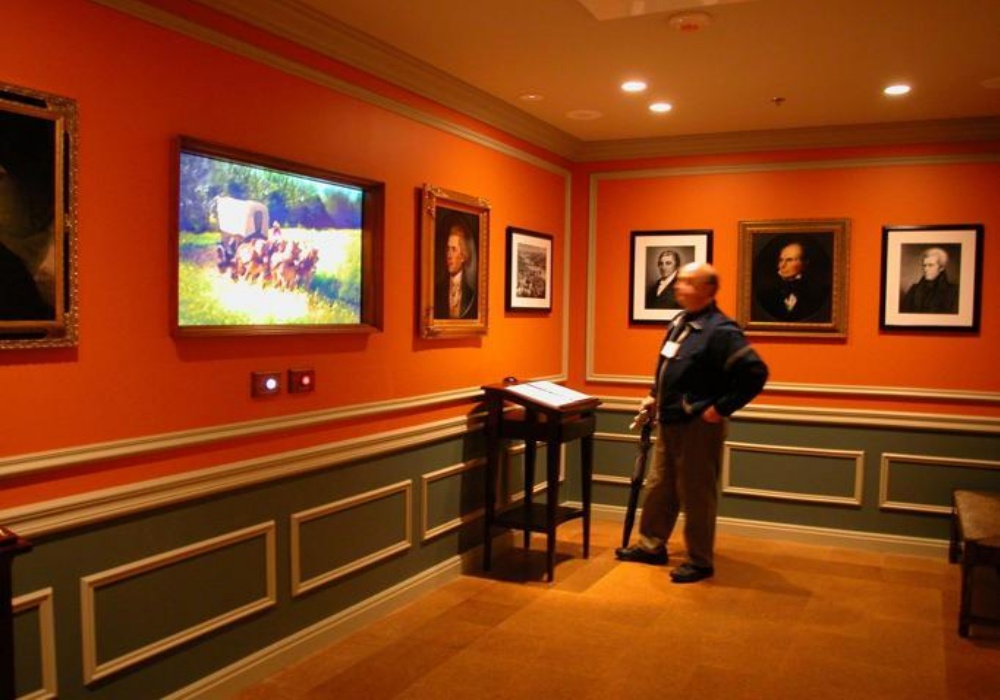}
 \incpr{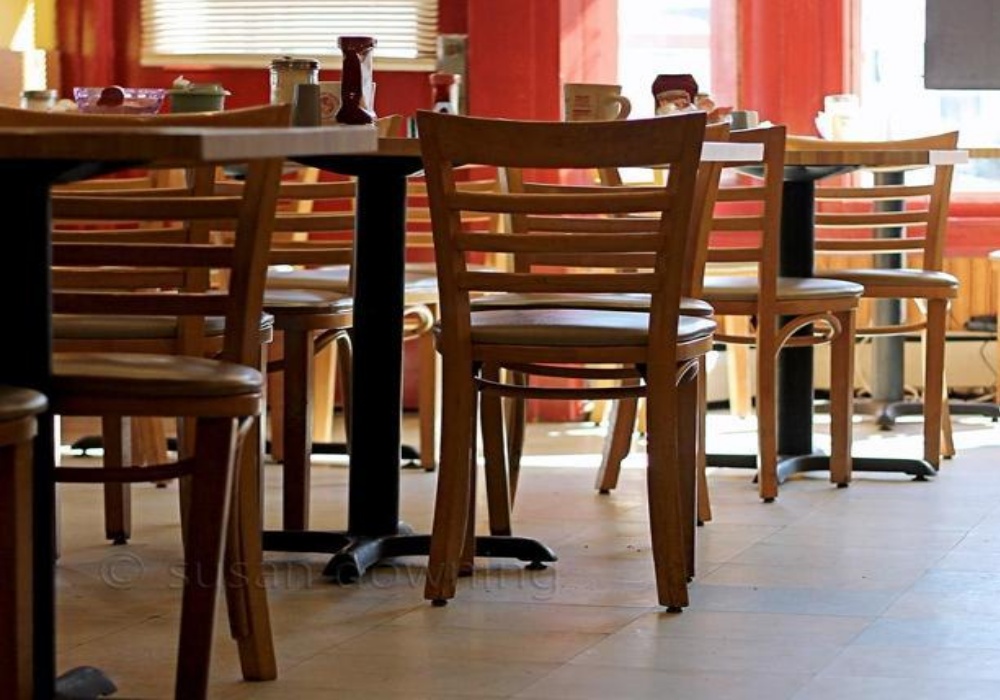}
\end{subfigure}\\[1pt]
\begin{subfigure}{1.0\linewidth}\centering
 \rotatebox{90}{\begin{minipage}{.134\linewidth}\center\small ground truth\end{minipage}}
 \incpr{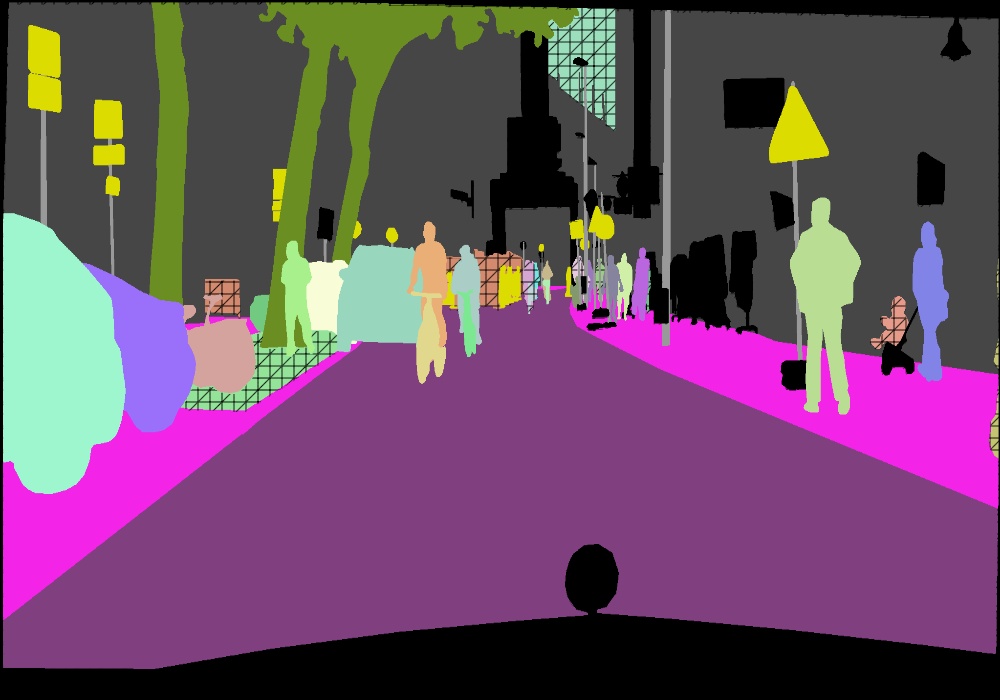}
 \incpr{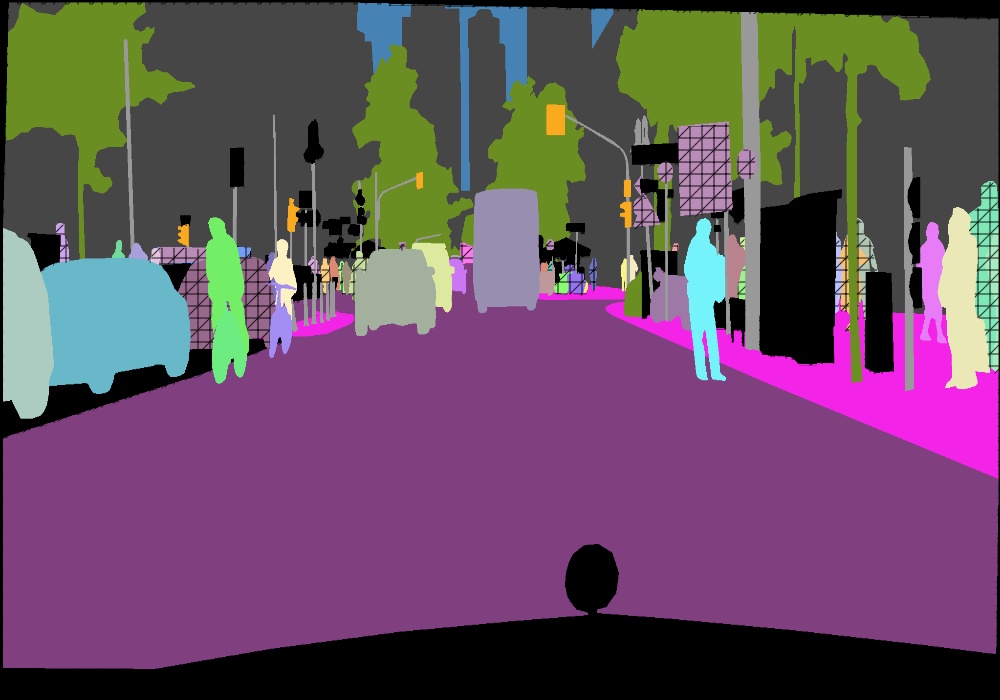}
 \incpr{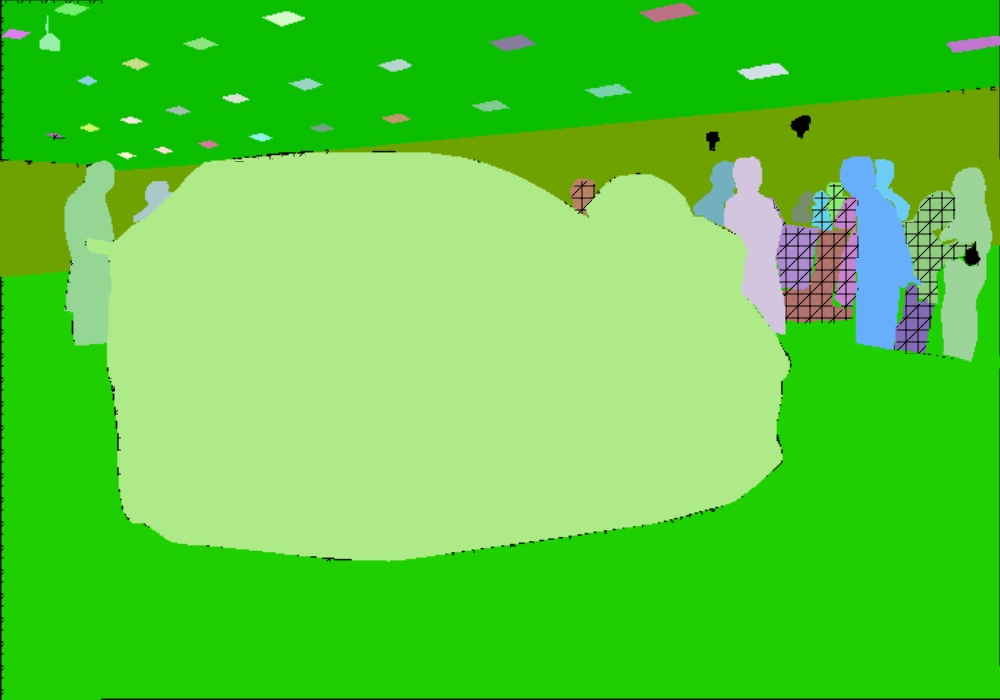}
 \incpr{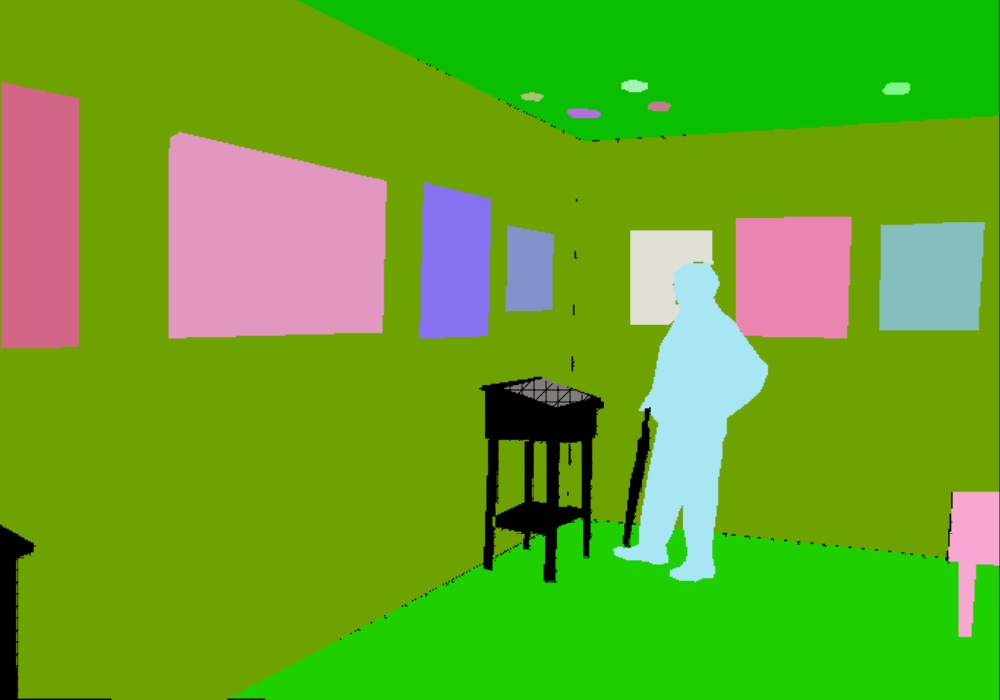}
 \incpr{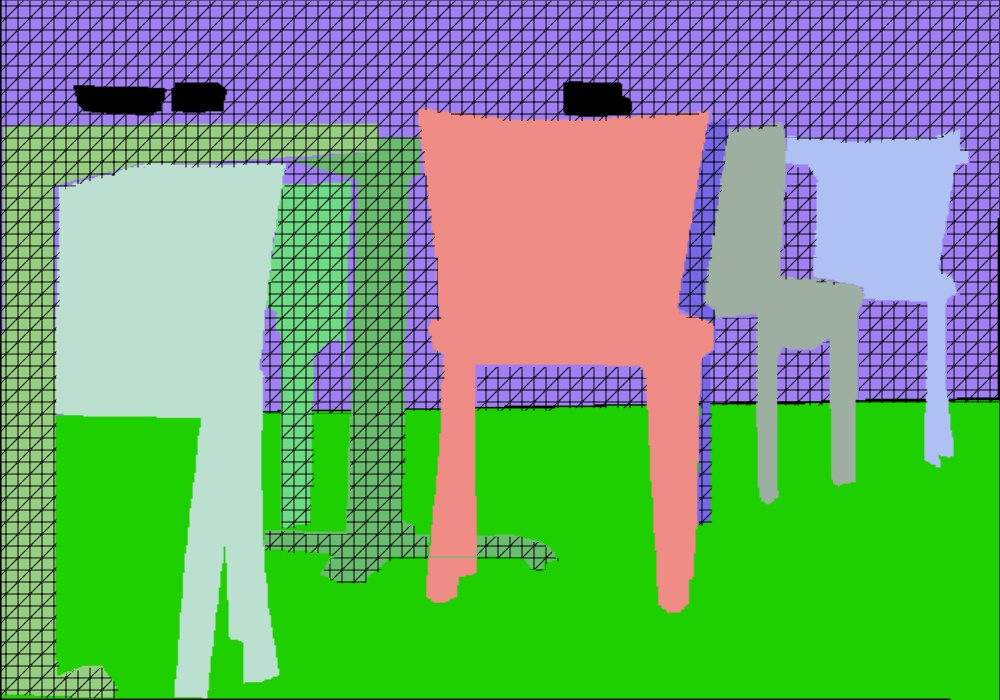}
\end{subfigure}\\[1pt]
\begin{subfigure}{1.0\linewidth}\centering
 \rotatebox{90}{\begin{minipage}{.134\linewidth}\center\small prediction\end{minipage}}
 \incpr{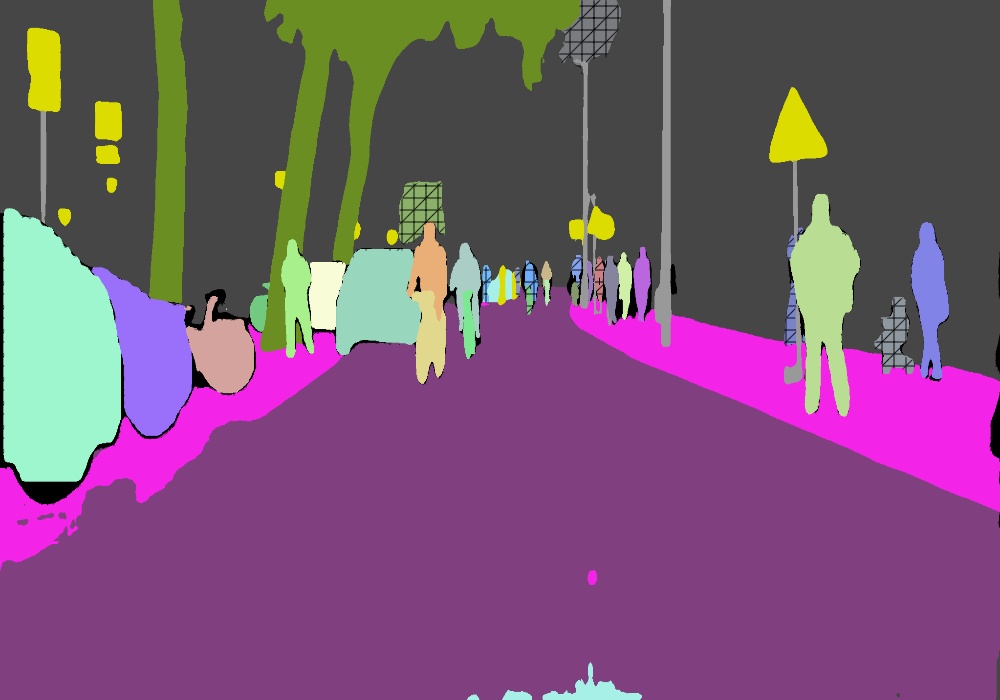}
 \incpr{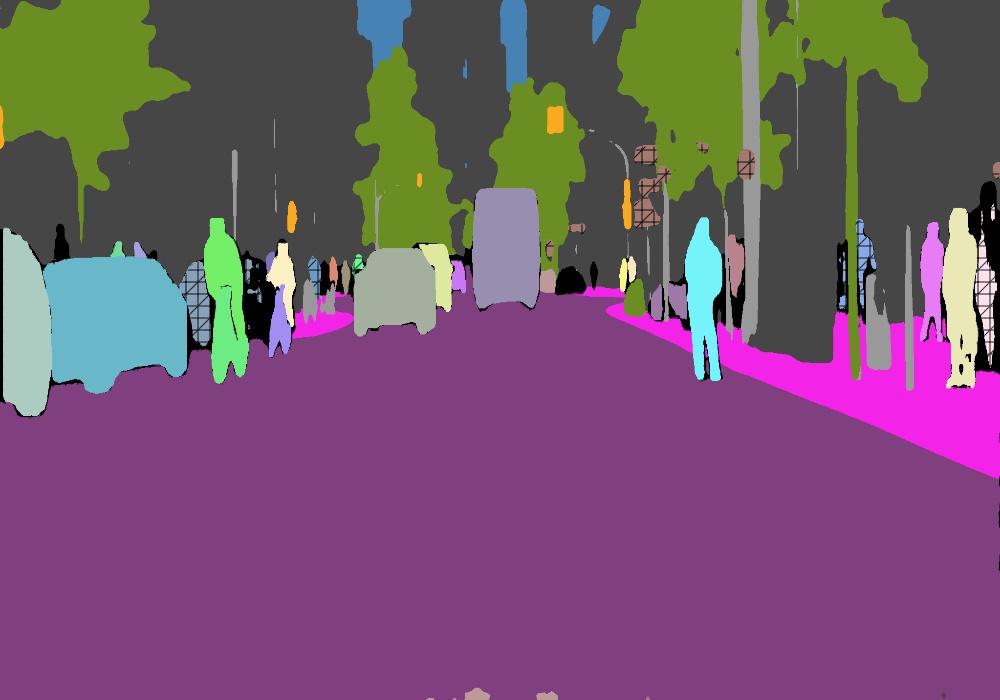}
 \incpr{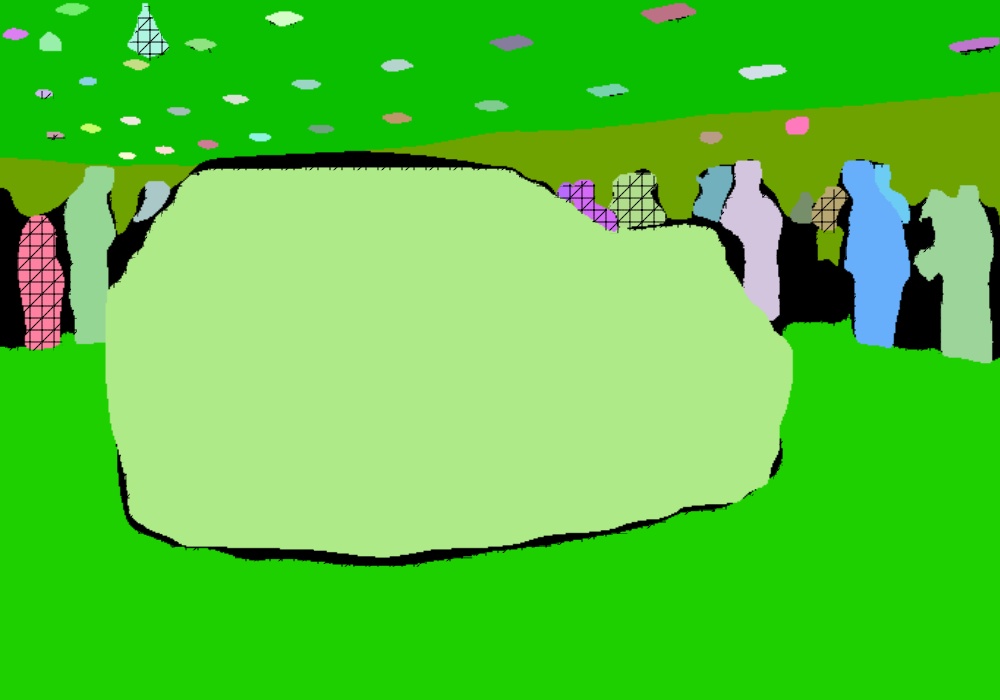}
 \incpr{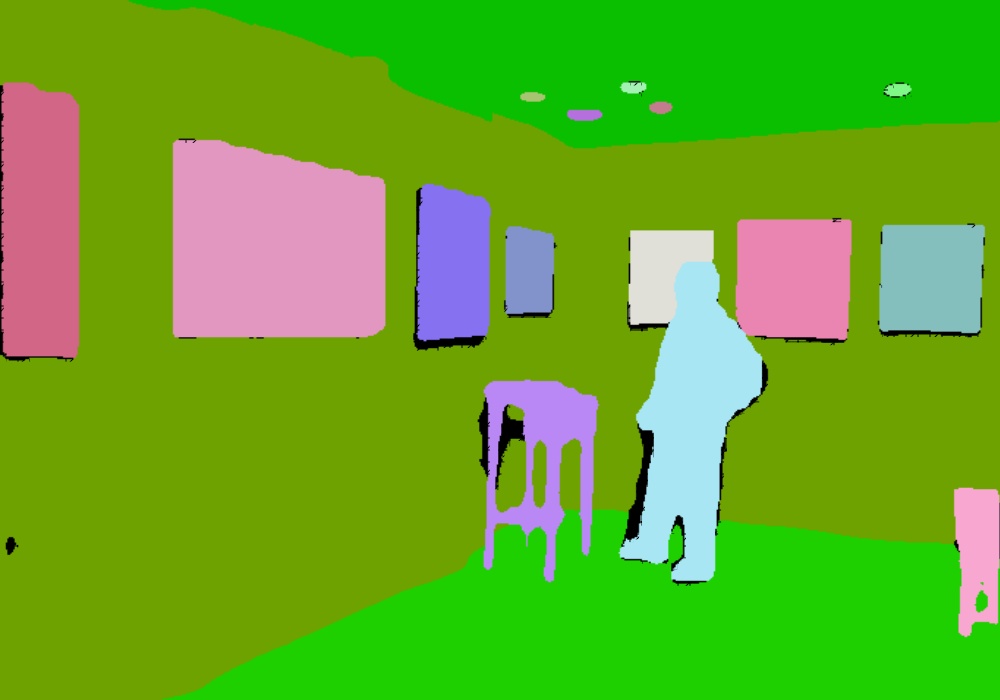}
 \incpr{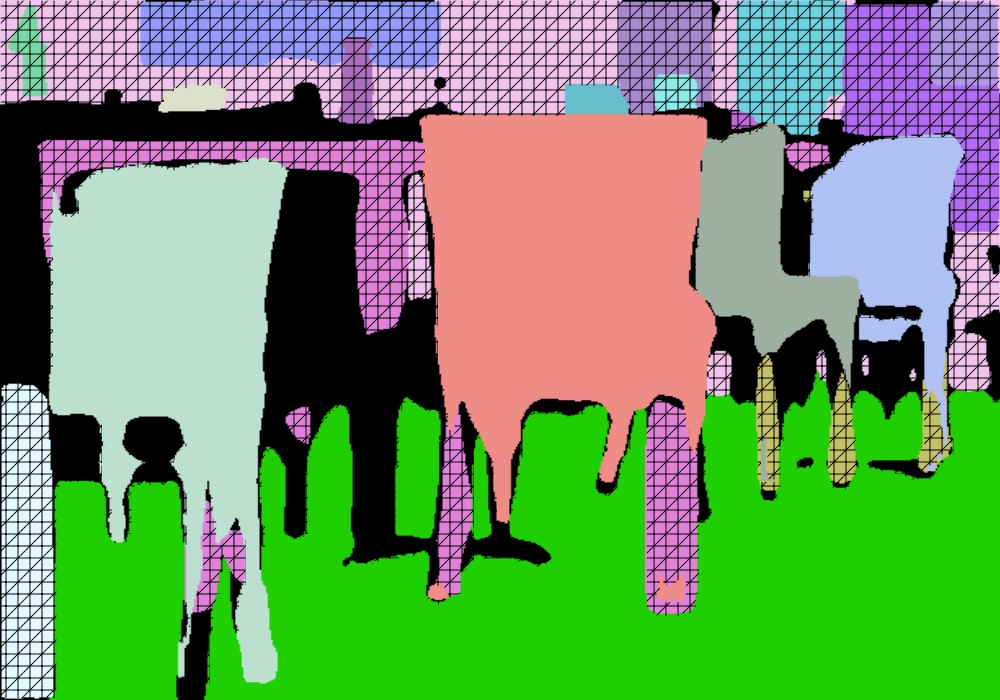}
\end{subfigure}\vspace{-2mm}
\caption{\textbf{Panoptic segmentation results} on Cityscapes (left two) and ADE20k (right three). Predictions are based on the merged outputs of state-of-the-art instance and semantic segmentation algorithms (see Tables \ref{tab:machines-instance} and \ref{tab:machines-semantic}). Colors for matched segments (IoU$>$0.5) match (crosshatch pattern indicates unmatched regions and black indicates unlabeled regions). Best viewed in color and with zoom.}
\label{fig:examples}
\end{figure*}

\begin{table}
\tablestyle{18pt}{1.05}\resizebox{!}{20mm}{
\begin{tabular}{@{}l|ccc@{}}
 {\bf Cityscapes} & PQ & PQ\stuff & PQ\things\\\shline
 machine-separate & n/a  & 66.6 & 54.0\\
 machine-panoptic & 61.2 & 66.4 & 54.0\\
\multicolumn{4}{c}{\vspace{-0.8em}}\\
 {\bf ADE20k} & PQ & PQ\stuff & PQ\things\\\shline
 machine-separate & n/a  & 27.4 & 41.1\\
 machine-panoptic & 35.6 & 24.5 & 41.1\\
\multicolumn{4}{c}{\vspace{-0.8em}}\\
 {\bf Vistas} & PQ & PQ\stuff & PQ\things\\\shline
 machine-separate & n/a  & 43.7 & 35.7\\
 machine-panoptic & 38.3 & 41.8 & 35.7
\end{tabular}}\vspace{-1mm}
\caption{\textbf{Panoptic \vs independent predictions.} The `machine-separate' rows show PQ of semantic and instance segmentation methods computed independently (see also Tables \ref{tab:machines-instance} and \ref{tab:machines-semantic}). For `machine-panoptic', we merge the non-overlapping thing and stuff predictions obtained from state-of-the-art methods into a true panoptic segmentation of the image. Due to the merging heuristic used, PQ\things stays the same while PQ\stuff is slightly degraded.}
\label{tab:separate-vs-panoptic}\vspace{-1mm}
\end{table}

\section{Future of Panoptic Segmentation}

Our goal is to drive research in novel directions by inviting the community to explore the new panoptic segmentation task. We believe that the proposed task can lead to expected and unexpected innovations. We conclude by discussing some of these possibilities and our future plans.

Motivated by simplicity, the PS `algorithm' in this paper is based on the \emph{heuristic} combination of outputs from top-performing instance and semantic segmentation systems. This approach is a basic first step, but we expect more interesting algorithms to be introduced. Specifically, we hope to see PS drive innovation in at least two areas: (1) Deeply integrated end-to-end models that simultaneously address the dual stuff-and-thing nature of PS. A number of instance segmentation approaches including \cite{liu2017sgn,arnab2017pixelwise,bai2016deep,kirillov2016instancecut} are designed to produce non-overlapping instance predictions and could serve as the foundation of such a system. (2) Since a PS cannot have overlapping segments, some form of higher-level `reasoning' may be beneficial, for example, based on extending learnable NMS \cite{desai2011discriminative,hosang2017learning} to PS. We hope that the panoptic segmentation task will invigorate research in these areas leading to exciting new breakthroughs in vision.

Finally we note that the panoptic segmentation task was featured as a challenge track by both the COCO \cite{lin2014coco} and Mapillary Vistas \cite{neuhold2017mapillary} recognition challenges and that the proposed task has already begun to gain traction in the community (\eg \cite{li2018attention,xiong2019upsnet,yang2019deeperlab,liu2019end,li2018weakly,li2018learning,kirillov2019panopticfpn} address PS).

\begin{table}
\tablestyle{6pt}{1.05}\resizebox{!}{20mm}{
\begin{tabular}{@{}l|lll|ll@{}}
{\bf Cityscapes}
 & PQ & SQ & RQ & PQ\stuff & PQ\things\\\shline
 human & 69.6\tpm{2.5}{2.7} & 84.1\tpm{0.8}{0.8} & 82.0\tpm{2.7}{2.9} & 71.2\tpm{2.3}{2.5} & 67.4\tpm{4.6}{4.9}\\
 machine & 61.2 & 80.9 & 74.4 & 66.4 & 54.0\\
\multicolumn{6}{c}{\vspace{-0.8em}}\\
{\bf ADE20k} & PQ & SQ & RQ & PQ\stuff & PQ\things\\\shline
 human & 67.6\tpm{2.0}{2.0} & 85.7\tpm{0.6}{0.6} & 78.6\tpm{2.1}{2.1} & 71.0\tpm{3.7}{3.2} & 66.4\tpm{2.3}{2.4}\\
 machine & 35.6 & 74.4 & 43.2 & 24.5 & 41.1\\
\multicolumn{6}{c}{\vspace{-0.8em}}\\
{\bf Vistas} & PQ & SQ & RQ & PQ\stuff & PQ\things\\\shline
 human & 57.7\tpm{1.9}{2.0} & 79.7\tpm{0.8}{0.7} & 71.6\tpm{2.2}{2.3} & 62.7\tpm{2.8}{2.8} & 53.6\tpm{2.7}{2.8}\\
 machine & 38.3 & 73.6 & 47.7 & 41.8 & 35.7
\end{tabular}}\vspace{-1mm}
\caption{\textbf{Human \vs machine performance.} On each of the considered datasets human consistency is much higher than machine performance (approximate comparison, see text for details). This is especially true for RQ, while SQ is closer. The gap is largest on ADE20k and smallest on Cityscapes. Note that as only a small set of human annotations is available, we use bootstrapping and show the the 5\tss{th} and 95\tss{th} percentiles error ranges for human results.}
\label{tab:human-vs-machines}\vspace{-1mm}
\end{table}

\linespread{0.95}{\footnotesize\bibliographystyle{style.files/ieee}\bibliography{panoptic}}

\end{document}